\documentclass[letterpaper]{article} % DO NOT CHANGE THIS
\usepackage{aaai2026}  % DO NOT CHANGE THIS
\usepackage{times}  % DO NOT CHANGE THIS
\usepackage{helvet}  % DO NOT CHANGE THIS
\usepackage{courier}  % DO NOT CHANGE THIS
\usepackage[hyphens]{url}  % DO NOT CHANGE THIS
\usepackage{graphicx} % DO NOT CHANGE THIS
\usepackage{natbib}  % DO NOT CHANGE THIS AND DO NOT ADD ANY OPTIONS TO IT
\usepackage{caption} % DO NOT CHANGE THIS AND DO NOT ADD ANY OPTIONS TO IT
\usepackage{algorithm}
\usepackage{algorithmic}

\usepackage{newfloat}
\usepackage{listings}
\DeclareCaptionStyle{ruled}{labelfont=normalfont,labelsep=colon,strut=off} % DO NOT CHANGE THIS
\floatstyle{ruled}
\newfloat{listing}{tb}{lst}{}
\floatname{listing}{Listing}
\usepackage{booktabs}       % professional-quality tables
\usepackage{amsfonts}       % blackboard math symbols
\usepackage{nicefrac}       % compact symbols for 1/2, etc.
\usepackage{microtype}      % microtypography
\usepackage[table, dvipsnames]{xcolor}

\usepackage{enumitem}

\usepackage{multirow}
\usepackage{makecell}
\usepackage{threeparttable}

\usepackage{array}
\newcommand{\PreserveBackslash}[1]{\let\temp=\\#1\let\\=\temp}
\newcolumntype{C}[1]{>{\PreserveBackslash\centering}p{#1}}
\usepackage{hhline}
\newcolumntype{?}{!{\vrule width 1pt}}

\usepackage{amsmath}
\usepackage{amsthm}
\usepackage{amssymb}
\usepackage{thmtools}
\usepackage{thm-restate}
\usepackage[title]{appendix}
\usepackage{comment}
\usepackage{dsfont}
\usepackage{cleveref}

\usepackage{tikz}
\usepackage{mathtools}

\definecolor{darkblue}{rgb}{0,0,0.95}

\usepackage{subfigure}

\def\Regret{\mathrm{Reg}}

\newtheorem{lemma}{Lemma}

\newtheorem{claim}{Claim}
\newtheorem{remark}{Remark}

\newenvironment{proofsketch}{%
  \proof}{\endproof}
\newcommand{\Olog}{\tilde{\mathcal{O}}}

\newtheorem{theorem-rst}[theorem]{Theorem}
\newtheorem{lemma-rst}[lemma]{Lemma}
\newtheorem{proposition-rst}[lemma]{Proposition}
\newtheorem{assumption-rst}[lemma]{Assumption}
\newtheorem{claim-rst}[claim]{Claim}
\newtheorem{corollary-rst}[lemma]{Corollary}

\DeclarePairedDelimiter\br{(}{)}% ( )
\DeclarePairedDelimiter\brs{[}{]}% [ ]
\DeclarePairedDelimiter\brc{\{}{\}}% { }

\DeclarePairedDelimiter\abs{\lvert}{\rvert}% | |
\DeclarePairedDelimiter\norm{\lVert}{\rVert}% || ||

\DeclarePairedDelimiter\ceil{\lceil}{\rceil}% || ||

\DeclareMathOperator*{\argmin}{arg\,min}

\newcommand{\E}{\mathbb{E}}
\newcommand{\R}{\mathbb{R}}
\newcommand{\G}{\mathbb{G}}
\newcommand{\N}{\mathbb{N}}

\newcommand{\F}{\mathcal{F}}

\newcommand{\X}{\mathcal{X}}

\newcommand{\Gcal}{\mathcal{G}}
\newcommand{\Ncal}{\mathcal{N}}
\newcommand{\Ocal}{\mathcal{O}}
\newcommand{\Pb}{\mathbb{P}}
\newcommand{\Pcal}{\mathcal{P}}

\newcommand{\Ind}[1]{\mathds{1}\brc*{#1}}

\newcommand{\xhat}{\hat{x}}
\newcommand{\yhat}{\hat{y}}

\newcommand{\kl}{\mathrm{KL}}
\newcommand{\klBin}{\mathrm{kl}}

\renewcommand{\tilde}{\widetilde}

\usepackage[colorinlistoftodos]{todonotes}

\def\showComments{} %% set to true - show comments
\ifdefined\showComments
    \newcommand{\comN}[1]{\textcolor{blue}{\{Nadav: #1\}}}
    \newcommand{\comK}[1]{\textcolor{olive}{\{Kyoungseok: #1\}}}
    \newcommand{\comA}[1]{\textcolor{red}{\{AAA: #1\}}}
\else
    \newcommand{\comN}[1]{}
    \newcommand{\comK}[1]{}
    \newcommand{\comA}[1]{}
\fi

\usepackage{selectp}
\title{Online Linear Regression with Paid Stochastic Features}
\author {
    Nadav Merlis\textsuperscript{\rm 1},
    Kyoungseok Jang\textsuperscript{\rm 2},
    Nicolò Cesa-Bianchi\textsuperscript{\rm 3}
}
\affiliations {
    \textsuperscript{\rm 1}Technion, 
    \textsuperscript{\rm 2}Chung-Ang University,
    \textsuperscript{\rm 3}University of Milan \& Politecnico di Milano\\
    nmerlis@technion.ac.il, ksjang@cau.ac.kr, nicolo.cesa-bianchi@unimi.it 
}

\begin{document}

\maketitle

\begin{abstract}
We study an online linear regression setting in which the observed feature vectors are corrupted by noise and the learner can pay to reduce the noise level. In practice, this may happen for several reasons: for example, because features can be measured more accurately using more expensive equipment, or because data providers can be incentivized to release less private features. Assuming feature vectors are drawn i.i.d.\ from a fixed but unknown distribution, we measure the learner's regret against the linear predictor minimizing a notion of loss that combines the prediction error and payment. When the mapping between payments and noise covariance is known, we prove that the rate $\sqrt{T}$ is optimal for regret if logarithmic factors are ignored. When the noise covariance is unknown, we show that the optimal regret rate becomes of order $T^{2/3}$ (ignoring log factors). Our analysis leverages matrix martingale concentration, showing that the empirical loss uniformly converges to the expected one for all payments and linear predictors.
\end{abstract}

\section{Introduction}

In online linear regression, a learner sequentially observes i.i.d.\ feature vectors $x_t\in\R^d$ and aims to correctly predict outputs $y_t=x_t^\top\theta^*$, knowing neither $\theta^*$ nor the distribution from which the feature vectors are drawn. When measured by the quadratic prediction error, this is potentially the most canonical online regression setting, and several extensions to more complex losses and regression models have been studied in the past. Although in the vast majority of works the features $x_t$ are assumed to be perfectly observable, in practice, we usually only have access to a noisy version of $x_t$. For example, features obtained from experiments often have measurement errors, and contextual information on users only relies on a random sample of their behavior. Sometimes the features are even intentionally corrupted (e.g., via locally differentially private mechanisms) to protect user privacy.

Motivated by scenarios where the learner can reduce the noise intensity by allocating more resources or more work, we study a variant of online linear regression where the features are stochastic and corrupted by additive noise, and the learner may decide to pay to reduce the noise level so that a higher payment corresponds to a smaller noise covariance (w.r.t.\ the positive definite order). When conducting experiments, this payment can be related to renting better lab equipment or hiring more experienced staff, but also to extending the measurement duration and/or averaging more raw samples. With user data, this noise reduction can be achieved by collecting more statistics before making a decision. In privacy-protected settings, the learner may offer payments to compensate users for limited privacy loss (either directly or by offering free services)---see, e.g., \citep{wang2018value}. While reducing the feature noise naturally facilitates prediction, it also causes some penalty to the learner (either time, money, or given services). Thus, the learner should optimally balance the prediction error and the cost due to payments.

We measure a predictor's performance by a loss combining the prediction error with the payment. The goal of the learner is to minimize regret, defined as the difference between the learner's expected cumulative loss and the cumulative loss of the best linear predictor, i.e., the one that optimally balances prediction errors and payments. 

We first tackle the setting where the learner knows how the noise covariance changes as a function of the payment. This can be the case, for instance, when the noise is generated by a known differentially private mechanism. We show how to use this knowledge to effectively utilize samples collected at one payment level to estimate the optimal prediction at any other payment level. This information sharing leads to a major acceleration in the learning process and obviates the need for payment exploration. In particular, we devise an empirical loss estimator that uniformly converges to the expected one across all potential payment levels and predictions (including payments that were never offered by the algorithm). We utilize this result to prove that an algorithm that greedily minimizes this empirical loss to determine the payment level and linear predictor achieves a regret bound of $\Olog\br{\sqrt{T}}$ after $T$ prediction rounds. We complement this result with a lower bound, proving that the dependence on $T$ in this bound can only be improved by logarithmic factors. Note that the optimal rate in this setting is exponentially worse than the $\Theta(\ln T)$ bound for the noise-free case. 

We then move on to the case where the noise covariances are not known in advance. This could happen, for example, when noise reduction is the result of an extended data collection period, especially if the noise is temporally correlated. 
In this scenario, we can no longer use our uniform estimation scheme and must incorporate a payment exploration mechanism into our algorithm. Inspired by UCB-like algorithms \citep{auer2002using}, we limit the potential payments to a discretized grid and introduce an optimistic loss estimator, which encourages the algorithm to choose less-played payment levels. We show that by carefully tuning the grid size, this approach yields a regret bound of $\Olog\br*{T^{2/3}}$. We prove that, similarly to the case of known covariances, the dependence on $T$ in this bound can only be improved by logarithmic factors. Our lower bound construction carefully designs a set of noise covariance profiles to induce losses that imitate hard instances in Lipschitz bandits \citep{kleinberg2004nearly}.

\paragraph{Outline.} The rest of the paper is organized as follows. We first position our work in light of the existing literature in \Cref{section: related work}. We then formalize our setting and describe our assumptions in \Cref{section: setting}. In \Cref{section: known noise covariance}, we present our loss estimator when the noise covariance profile is known in advance; we show how to use this estimator to derive an order-optimal algorithm. Then, \Cref{section: unknown noise covariance} extends these results to situations when the noise covariance profile is not given to the learner. We conclude the paper with a summary and a discussion on the many potential future directions in \Cref{section: summary}.

\section{Related Work}
\label{section: related work}
\paragraph{Errors-in-variables models.} Linear regression with noisy covariates (also known as errors-in-variables, or EIV) is a classical topic in statistics---see, e.g., \citep{fuller2009measurement} and \citep{agarwal2021causal} for recent applications to differential privacy. Most results in EIV concern the problem of estimating $\theta^*$ under different norms, typically in a high-dimensional setting (large $d$) under assumptions of sparsity ($\theta^*$ has few non-zero components) or low-rank (the covariates admit a low-dimensional representation), see \citep{rosenbaum2010sparse,loh2011high,chen2013noisy,rosenbaum2013improved,kaul2015weighted,belloni2017linear,datta2017cocolasso} and references therein. A related line of work studies EIV under Bayesian assumptions \citep{reilly1981bayesian,ungarala2000multiscale,figueroa2022robust}. \citet{agarwal2023adaptive} consider EIV in an estimation error task where covariates are chosen adaptively (as in active learning or in multi-armed bandits). Closer to our work is the paper by \citet{agarwal2019robustness}, where they study a linear regression model with fixed design and missing or corrupted variables. Unlike us, they bound the prediction error in a transductive setting, where the covariates on which predictions are evaluated are available at training time. 

\paragraph{Local DP.} A prominent situation where features are noisy is when they are protected by a local differentially private mechanism \citep{kasiviswanathan2011can}. Two well-known approaches to create this protection are to add the raw features either Laplace or Gaussian noise \citep{yang2024local}. Many works study how to learn when features are private, including in the context of linear regression \citep{duchi2013local,smith2017interaction,wang2019sparse,fukuchi2017differentially}. However, their goal is mostly to estimate the true parameters of the model given the noisy features; in contrast, we want to learn how to use online noisy data to output good predictions. As we will show, the best predictors are not the true parameters, but rather depend on the noise level.

\paragraph{Online learning with noisy features.} Closely related to ours is the work by \citet{cesa2011online}, who studied online linear regression with square loss with noisy features and known/unknown covariances. However, there are several crucial differences: they do not have a notion of cost, and their notion of regret compares to the best linear predictor on the noise-free features. Moreover, their focus is on a setting in which the learner may obtain more than one noisy realization of the same feature vector. A second closely related work is the one by \citet{van2023trading}. They prove a $d^2(\ln T)\sqrt{T}$ regret bound in a binary classification setting in which: the learner can pay to reduce the noise level, the covariates (called experts in their setting) are binary, and noise and payments apply independently on each coordinate. However, because of their focus on classification with independent feature noise, their techniques are not directly applicable to our regression setting with correlated noise. Other works assume noisy features, but in the partial feedback setting, include \citep{kim2023contextual} and \citep{zheng2020locally}.

\paragraph{Privacy and incentives.} Another aspect of our work is the learner's capability to pay for noise reduction. In the privacy literature, this is often addressed as a problem of mechanism design: an analyst designs a procedure that incentivize individuals to sell their private data, aiming to obtain the most accurate estimation at a fixed payment budget (or paying a minimal cost for a target accuracy level) \citep{ghosh2011selling,nissim2012privacy,nissim2014redrawing,wang2018value,fallah2024optimal}. \citet{hsu2014differential} also studied the problem in which an individual is offered a payment to participate in a study and is willing to join it only if the payment covers its privacy loss. They show how an analyst running the experiment should tune the payment and privacy levels according to various considerations, including accuracy. All these works focus on payments that minimize the estimation error of the true parameters. Instead, our goal is to choose payments that balance the prediction error and the feature costs for every new incoming user. Also, our payment model is different: in our setting, the noise level is fixed (and potentially unknown) given the payment.

\section{Setting}
\label{section: setting}
We study the problem of online linear regressions when the learner observes noisy features and can pay to reduce the noise level. At each round $t\in[T]$, the environment privately generates the true features $x_t\in\R^d$ independently from a fixed unknown distribution $D_x$ and the output $y_t = x_t^\top \theta^*$ for some unknown $\theta^*\in\R^d$.\footnote{All results can be easily extended to noisy linear outputs, as we discuss in \Cref{remark: additive noise} and \Cref{appendix: additive noise}.} The learner decides to pay a cost (`payment') $c_t\in[0,1]$ and in return observes noisy features $\xhat_t(c_t)= x_t +n_t(c_t)$, where $n_t(c)\sim D_n(c)$ is a zero-mean random variable independent of $x_t$. 
For brevity, we often omit the dependence on the cost $c_t$ and denote the noisy features by $\xhat_t$. Given this observation, the learner tries to predict $\yhat_t(\xhat_t)$ so that it is as close as possible (in the squared distance) to $y_t$, based on the noisy features $\xhat_t$ and the data observed in the previous rounds. After the prediction, the learner observes the real value $y_t$. We focus on the class of linear predictors $\yhat_t(\xhat_t)=\xhat_t^\top\nu$. We denote $\E\brs*{x_t}=\bar{x}$, and also denote the covariance matrices of samples from $D_x$ and $D_n(c)$ by $\Sigma_x=\E\brs*{x_tx_t^\top}$ and $\Sigma_n(c)=\E\brs*{n_t(c)n_t(c)^\top}$, respectively. Lastly, we denote the covariance of $\xhat_t$ given a cost $c$ by $\Sigma_{\xhat}(c) =\E\brs*{\xhat_t(c)\xhat_t(c)^\top}= \Sigma_x + \Sigma_n(c)$.\footnote{Since $x_t,\xhat_t$ are not zero mean, $\Sigma_x,\Sigma_{\xhat}(c)$ are sometimes called their autocorrelation.}

\paragraph{Additional Notations.} We define the quadratic form of a vector $x\in\R^d$ w.r.t.\ a positive (semi-)definite matrix $A\in\R^{d\times d}$ as $\norm*{x}_A^2 = x^\top A x$. We also denote by $\norm*{x}$, the $\ell_2$ norm of a vector $x\in\R^d$, and by $\norm*{A}_{op}=\sup_{x\ne 0}\frac{\norm*{Ax}}{\norm*{x}}$, the operator norm of a matrix $A\in\R^{d\times d}$. We use $A\succ0$ to denote positive definite matrices, and similarly write $A\succeq0$ if $A$ is positive semi-definite and $B\succeq A$ when $B-A\succeq0$. We denote the observed history of the decision process up to the beginning of round $t+1$ (including the choice of payment but not the features) by $I_t^o = \brc*{c_1,\xhat_1,\yhat_1,y_1,\dots, c_t,\xhat_t,\yhat_t,y_t,c_{t+1}}$ and the entire past data by $I_t=I_t^o\cup\brc*{x_s,n_s}_{s\in[t]}$.

\paragraph{Regularity assumptions.} We assume that increasing the payment decreases the noise level, i.e., if $c_1\le c_2$, it holds that $\Sigma_n(c_2)\preceq \Sigma_n(c_1)$. Intuitively, it is equivalent to requiring that the variance of the noise measured in any arbitrary direction cannot rise when we increase the payment. We also assume that $\Sigma_{\xhat}(1)\succ 0$; by the covariance monotonicity assumption, this also implies that $\Sigma_{\xhat}(c)\succ 0$ for all $c\in[0,1]$. We further assume that both the features $x_t$ and the noise $n_t(c)$ are conditionally $R^2$-subgaussian random vectors, that is, for any $u\in\R^d$, we have $\E\brs*{e^{u^\top(x_t-\bar{x})}\vert I_{t-1}} \leq e^{\frac{\norm{u}^2R^2}{2}}$ and $\E\brs*{e^{u^\top n_t}\vert I_{t-1}} \leq e^{\frac{\norm{u}^2R^2}{2}}.$
% {\small\begin{align*}
%     &\E\brs*{e^{u^\top(x_t-\bar{x})}\vert I_{t-1}} \leq e^{\frac{\norm{u}^2R^2}{2}},\quad \textrm{and} \\
%     &\E\brs*{e^{u^\top n_t}\vert I_{t-1}} \leq e^{\frac{\norm{u}^2R^2}{2}}.
% \end{align*}}
Finally, we assume that both $\norm*{\theta^*}$ and $\norm*{\bar{x}}$ are upper bounded by a known $S>0$.

\paragraph{Optimality Criterion.} We compare ourselves to a linear regressor which pays a cost that optimally balances the prediction error and the payments for the features. In particular, for a given $\lambda>0$, we define the loss for a fixed cost level $c$ and linear predictor $\nu$ as
\begin{align*}
    \ell(c,\nu) = \E_{x\sim D_x, n\sim D_n(c)}\brs*{\br*{\br*{x+n}^\top\nu - x^\top\theta^*}^2} + \lambda c.
\end{align*}
With a slight abuse of notation, for any given $S>0$, we define the optimal linear predictor and loss given cost $c$ as 
\begin{align*}
    \nu^*(c) = \argmin_{\nu: \norm*{\nu}\le S} \brc*{ \ell(c,\nu)}, \qquad
    \ell^*(c) = \min_{\nu: \norm*{\nu}\le S}\brc*{ \ell(c,\nu)}, 
\end{align*}
and denote the optimal cost and loss by
\begin{align*}
    & c^* \in \argmin_{c\in[0,1]} \brc*{ \ell^*(c)},\\ 
    &\ell^* = \min_{\nu: \norm*{\nu}\le S,c\in[0,1]}\brc*{ \ell(c,\nu)} = \min_{c\in[0,1]}\brc*{ \ell^*(c)}.
\end{align*}
We remark that $\nu^*(c),\ell^*(c)$ and $\ell^*$ all depend on $S$; we omit this dependence to simplify notations.
\begin{remark}
    \label{remark: linear}
    Our choice to limit ourselves to linear predictions is due to both computational and statistical constraints. Specifically, a nonlinear optimal predictor requires computing $\E\brs*{x_t\vert \xhat_t}^\top\theta^*$. This estimator heavily depends on the entire feature and noise distributions, rendering the estimation problem extremely complicated and usually requires intractable computations. We adopt the solution of \citet{van2023trading} to this issue, and instead compare ourselves to the natural tractable benchmark. We also note that for some distributions, linear predictors are the best possible benchmark, as we later prove through our lower bounds.
\end{remark}
The loss and the optimal linear predictor can also be specified using the covariance matrices as follows:
\begin{restatable}{claim-rst}{optimalLoss}
    \label{claim: opt values}
    It holds that 
    \begin{align*}
    \ell(c,\nu) = \nu^\top \Sigma_{\xhat}(c)\nu - 2\nu^\top \Sigma_x\theta^* + {\theta^*}^\top \Sigma_x{\theta^*} +\lambda c .
    \end{align*}
    Moreover, let $\bar{\nu}(c)=\br*{\Sigma_{\xhat}(c)}^{-1}\Sigma_x\theta^*$. If $\norm*{\bar{\nu}(c)}\leq S$ then $\nu^*(c) = \bar{\nu}(c)$ and 
    \begin{align*}
        &\ell^*(c) = {\theta^*}^\top \Sigma_x\theta^* -{\theta^*}^T\Sigma_x\Sigma_{\xhat}(c)^{-1}\Sigma_x\theta^* + \lambda c.
    \end{align*}
\end{restatable}
Since we do not assume that the noise covariance is Lipschitz in the payment, the losses might not be continuous. Nonetheless, the monotonicity assumption induces a near-Lipschitz behavior:
\begin{restatable}{claim-rst}{lossProperties}
\label{claim: lipschitzness}
The loss $\ell^*(c)$ is $\lambda$-one-sided Lipschitz: for any $0\leq c_1\leq c_2\leq 1$, it holds that $\ell^*(c_2) \leq \ell^*(c_1) + \lambda(c_2-c_1)$.
\end{restatable}
In words, the loss cannot significantly increase from a slight rise in the payment. Intuitively, this holds since increasing the cost always reduces the prediction error; therefore, the loss increment is always upper bounded by the increase in the payment term $\lambda c$. 
This only linearly affects the loss, thus implying one-sided Lipschitzness. As a consequence, our algorithms could limit themselves to a finite grid of costs, knowing that even if the optimal cost is not on the grid, there is a slightly higher payment of similar loss \citep{tullii2024improved}. The proofs for both claims can be found at \Cref{appendix: properties}.
\begin{remark}
    Monotonicity of the noise covariances and Lipschitzness of the noise covariances are two incomparable assumptions -- neither implies the other. Nonetheless, we rely on the monotonicity only to prove \Cref{claim: lipschitzness}, as well as arguing that $\Sigma_{\xhat}(c)\succ 0$ for all $c\in[0,1]$. Since \Cref{claim: lipschitzness} also immediately holds if $\Sigma_n(c)$ is Lipschitz in $c$, under the assumption that $\Sigma_{\xhat}(c)\succ 0$, all our positive results would also hold for Lipschitz covariances.
\end{remark}

\subsection{Regret Definition}
As the objective, algorithms are similarly measured by a loss that combines their prediction error and feature payments. We focus on minimizing the expected cumulative loss, also equivalent to minimizing the regret:
\begin{align*}
    \Regret(T) = \E\brs*{\sum_{t=1}^T \br*{\br*{\yhat_t(\xhat_t) - y_t}^2 + \lambda c_t}} - T\ell^*.
\end{align*}
The regret can be decomposed as follows (see \Cref{appendix: properties} for the proof).
\begin{restatable}{lemma-rst}{regretDecomp}
    \label{lemma: regret decomposition}
    For any arbitrary history-dependent prediction rule $\yhat_t:\R^d\mapsto \R$ and any arbitrary sequence $\nu_1,\dots,\nu_T\in\R^d$, both potentially depend on $I_{t-1}^o$, it holds that 
    \begin{align*}
        \Regret(T) 
        & = \E\brs*{\sum_{t=1}^T \br*{\br*{\yhat_t(\xhat_t) - y_t}^2 - \br*{\br*{\xhat_t^\top\nu_t - y_t}^2}}}  \\
        &\quad+ \E\brs*{\sum_{t=1}^T \br*{\ell(c_t,\nu_t) - \ell^*}}.
    \end{align*}
    In particular, if the algorithm only outputs linear predictions $\yhat_t(\xhat_t) = \xhat_t^\top \nu_t$, then
    \begin{align*}
        \Regret(T) 
        & = \E\brs*{\sum_{t=1}^T \br*{\ell(c_t,\nu_t) - \ell^*}}.
    \end{align*}
\end{restatable}
The first term encapsulates the regret of the prediction rule compared to some dynamic linear predictor, while the second term measures this linear predictor compared to the best linear predictor. In the rest of the paper, we restrict ourselves to linear predictions (as discussed in \Cref{remark: linear}), and so we focus on bounding the loss difference.

\section{Online Linear Regression with Known Noise Covariances}
\label{section: known noise covariance}

In this section, we start from the favorable setting in which the learner has perfect knowledge of the noise covariances $\Sigma_n(c)$ for all $c\in[0,1]$. In particular, to output a good prediction at the right cost, the learner only misses the feature covariance $\Sigma_x$ and the linear map $\theta^*$. Since neither of these quantities depends on the payment $c_t$, our aim is to \emph{share information} between different cost observations, thus improving the regret. To this end, we take a closer look at the expected loss $\ell(c,\nu)$ in \Cref{claim: opt values}, noticing that only a single term depends on the noise covariance: a term of the form $\nu^\top \Sigma_{\xhat}(c)\nu$. Given noisy independent feature samples at cost $c$, this term can be estimated empirically via $\hat{\Sigma}_{\xhat}(c)=\sum_{s=1}^t\xhat_s(c)\xhat_s(c)^\top$, whose expectation is $t\Sigma_{\xhat}(c)$. However, we only observe the noisy features for payments we played $\brc{c_s}_{s\leq t}$, so we cannot directly use this estimator for other unplayed payments. To circumvent this, we utilize our knowledge of the noise covariance and include a correction term that \emph{shifts} the noise covariance from $c_s$ to $c$. In particular, we suggest estimating the covariance with
\begin{align*}
    \hat{\Sigma}_{\xhat}(c)=\sum_{s=1}^t\br*{\xhat_s(c_s)\xhat_s(c_s)^\top+ \Sigma_n(c)-\Sigma_n(c_s)}.
\end{align*}
This estimator uses \emph{all past samples} while enjoying the desired expectation 
\begin{align*}
    \E\brs*{\hat{\Sigma}_{\xhat}(c)}
    &=\E\brs*{\sum_{s=1}^t \br*{\xhat_s(c_s)\xhat_s(c_s)^\top + \Sigma_n(c)-\Sigma_n(c_s)}} \\
    &= \E\brs*{\sum_{s=1}^t \br*{\br*{\Sigma_x +\Sigma_n(c_s)} + \Sigma_n(c)-\Sigma_n(c_s)}} \\
    & = \E\brs*{\sum_{s=1}^t \br*{\Sigma_x + \Sigma_n(c)}}
    = t\Sigma_{\xhat}(c).
\end{align*}
Inspired by this, we add a correction term of the form $\nu^{\top}\br*{\Sigma_n(c)-\Sigma_n(c_s)}\nu$ to the natural empirical loss and define the following loss estimator:
{\small \begin{align*}
    \hat{L}^{kc}_t&(c,\nu)
    \!=\! \sum_{s=1}^t\br*{\!\br*{\xhat_s^{\top}\nu-y_s}^2 + \nu^{\top}\br*{\Sigma_n(c)-\Sigma_n(c_s)}\nu+\lambda c}\!.
\end{align*}}
We indeed show that this empirical loss uniformly converges to the expected loss $\ell(c,\nu)$ and is almost convex:
\begin{restatable}{proposition-rst}{lossEstKnownNoise}
    \label{prop: loss estimator known noise}
    With probability at least $1-3\delta$, for all $t\ge1$, $c\in[0,1]$ and $\nu\in\R^d$ s.t. $\norm*{\nu}\le S$, it holds that 
    \begin{align*}
        \abs*{\hat{L}^{kc}_t(c,\nu) - t\ell(c,\nu)} \leq  9S^2\sqrt{8t\beta_t^2\ln\frac{3dt(t+1)}{\delta}}
    \end{align*}
    where
    {\small\begin{align*}
        \beta_t &= R^2 \br*{d +2\sqrt{d\ln\frac{3t(t+1)}{\delta}}+2\ln\frac{3t(t+1)}{\delta}}  \\
        &\quad+ S^2 \br*{1 + 2\sqrt{\frac{\ln\frac{3t(t+1)}{\delta}}{d}}}
        = \Olog(R^2d +S^2).        
    \end{align*}}
    Furthermore, under the same event, the regularized loss    
    \begin{align}
        \label{eq:loss regularized known covariances}
        \hat{L}^{kc,R}_t(c,\nu) = \hat{L}^{kc}_t(c,\nu) + \gamma_t\norm*{\nu}^2
    \end{align}
    for $\gamma_t=2\sqrt{8t\beta_t^2\ln\frac{3dt(t+1)}{\delta}}$ is quadratic and strictly convex in $\nu$ for all $t\ge 1$.
\end{restatable}
\begin{proofsketch}
    As a first step, we show how to decompose the loss difference into the following quadratic terms:
    {\small\begin{align*}
    \hat{L}^{kc}_t(c,\nu) - t\ell(c,\nu)
    & =\nu^{\top}\br*{\sum_{s=1}^t\br*{\xhat_s\xhat_s^\top - \Sigma_x-\Sigma_n(c_s)}}\nu  \\
    &\quad + {\theta^*}^{\top}\br*{\sum_{s=1}^tx_s x_s^{\top} - \Sigma_x}\theta^*\\
    &\quad - 2\nu^{\top}\br*{\sum_{s=1}^tx_s x_s^{\top}-\Sigma_x}\theta^* \\
    &\quad - 2\nu^{\top}\br*{\sum_{s=1}^tn_s x_s^{\top}}\theta^*.
\end{align*}}
In particular, the first term already incorporates the covariance transfer from $\Sigma_n(c_s)$ to $\Sigma_n(c)$, leaving a well-behaved error term that depends on all the past payments, even when different costs were paid. 
The first three terms are quadratic forms w.r.t.\ the error between the empirical covariance of subgaussian vectors and their real covariance. The last term can also be written as such, noting that 
{\small\begin{align*}
    \nu^{\top}&n_s x_s^{\top}\theta^* 
    \!=\!  \begin{pmatrix} 0 \\ \nu \end{pmatrix}^\top\!\!\br*{\begin{pmatrix} x_s \\ n_s \end{pmatrix}\!\begin{pmatrix} x_s \\ n_s \end{pmatrix}^\top \!- \begin{pmatrix} \Sigma_x & 0 \\ 0 & \Sigma_n(c_s) \end{pmatrix}}\!\begin{pmatrix} \theta^* \\ 0 \end{pmatrix}\!.
\end{align*}}
Using the inequality $x^\top A y \leq \norm*{x}\norm*{y}\norm*{A}_{op}$, and given that the norms of $\nu,\theta^*$ are bounded, we can uniformly bound the loss difference by upper bounding the operator norms of the covariance error terms. To bound these norms, we adapt martingale matrix concentration results \citep{tropp2012user} and apply them to the empirical covariance of subgaussian vectors. In particular, if $Z_t$ is a conditionally subgaussian vector martingale of conditional covariance $\Sigma_t$, it holds with high probability that 
\begin{align*}
    \norm*{\sum_{s=1}^t Z_sZ_s^\top-\Sigma_s}_{op} < \gamma_t/2
\end{align*}
(see \Cref{lemma: matrix azuma for subgaussians} in \Cref{appendix: concentration} for the precise statement and proof). This leads to the uniform concentration of the loss. One direct implication is the strict convexity of the regularized loss -- noting that the loss is quadratic in $\nu$ and that $\xhat_t$ is $2R^2$-subgaussian, a regularization term of $\gamma_t\norm{\nu}^2$ is enough to turn the quadratic loss term into strictly convex.
\end{proofsketch}
\begin{remark}
    \label{remark: additive noise}
    The concentration bound of \Cref{prop: loss estimator known noise} can be easily extended to the case where the regression output is noisy, namely, $y_t=x_t^\top\theta^*+\eta_t$ for conditionally zero-mean subgaussian noise $\eta_t$. This extra noise does not affect the algorithm or the concentration rates; see further discussion in \Cref{appendix: additive noise}.
\end{remark}

\begin{algorithm}[t]
\caption{Online Regression with Paid Features -- Known Noise Covariances} \label{alg: known noise covariance}
\begin{algorithmic}
\STATE {\bf Require:} $\delta\in(0,1), K\in\N$
\STATE {\bf Initialize:} $\hat{L}^{kc,R}_0(c,\nu)=0$
\FOR{$t=1,...,T$}
    \FOR{$k=1,...,K$}
        \STATE Calculate $\hat{\nu}_{t-1}(k)\in \argmin_{\nu:\norm*{\nu}\leq S}\hat{L}^{kc,R}_{t-1}(k/K,\nu)$
    \ENDFOR
    \STATE Set $k_{t-1}\in\argmin_{k\in[K]}\hat{L}^{kc,R}_{t-1}(k/K,\hat{\nu}_{t-1}(k))$
    \STATE Pay $c_t=k_{t-1}/K$ and observe $\xhat_t = \xhat_t(c_t)$
    \STATE Predict $\yhat_t = \xhat_t^\top\hat{\nu}_{t-1}(k_{t-1})$ and observe $y_t$
    \STATE Update $\hat{L}^{kc,R}_{t}(k/K,\nu)$ for all $k\in[K]$% via \cref{eq:loss regularized known covariances}
\ENDFOR
\end{algorithmic}
\end{algorithm}

A natural algorithmic approach is to minimize the (regularized) empirical loss -- which we indeed do in \Cref{alg: known noise covariance}. Specifically, the algorithm minimizes this loss (w.r.t. the linear predictor $\nu$) for every cost on a dense discretized grid $c\in\brc*{k/K}_{k\in[K]}$ and pays the cost 
$$c_t\in\argmin_{c\in\brc*{k/K}_{k\in[K]}}\min_{\nu:\norm*{\nu}\leq S}\hat{L}^{kc,R}_t(c,\nu)$$
that achieves the minimal empirical loss on this grid. We note that the grid is only used to facilitate the optimal cost calculation; as long as it is sufficiently fine, it does not affect performance. Then, upon observing $\xhat_t(c_t)$, a prediction $\yhat_t(\xhat_t)=\xhat^\top\nu_t$ is calculated using the best empirical linear predictor for $c_t$:
$$ \nu_t = \hat{\nu}_{t-1}(c_t)\in\argmin_{\nu:\norm*{\nu}\leq S}\hat{L}^{kc,R}_t(c_t,\nu).$$
This algorithm enjoys the following performance bound (see \Cref{appendix: known covariances} for the proof):
\begin{restatable}{theorem-rst}{regretKnownNoiseCov}
    \label{theorem: regret known noise covariances}
    For any $T\ge1$, set $\delta=1/T$ and $K=\ceil*{\lambda T}$. Then, the regret of \Cref{alg: known noise covariance} is bounded by
\begin{align*}
    \Regret(T) = \Olog\br*{S^2(R^2d +S^2)\sqrt{T} + \lambda}.
\end{align*}
\end{restatable}
\begin{proofsketch}
For any $c$, the uniform concentration in \Cref{prop: loss estimator known noise} guarantees that the empirical loss $\hat{L}^{kc}_{t-1}$ is similar to the expected one $\ell$ both for $\hat{\nu}_{t-1}(c)$ and for $\nu^*(c)$. Since $\hat{\nu}_{t-1}(c)$ outperforms $\nu^*(c)$ on the empirical loss $\hat{L}^{kc}_{t-1}$, the loss similarity implies that it cannot be much worse than $\nu^*(c)$ on the expected loss $\ell$ for any payment $c\in[0,1]$. The suboptimality of $\hat{\nu}_{t-1}(c)$ compared to $\nu^*(c)$ will be proportional to the maximal error between the average empirical loss $\frac{1}{t}\hat{L}^{kc}_{t}$ and the true loss $\ell$: by \Cref{prop: loss estimator known noise}, it is of order $\Olog\br*{S^2\beta_t/\sqrt{t}}$.

As for the choice of payment, the information sharing enables us to collect data on all costs simultaneously, without the need for explicit exploration. Thus, we can choose the best payment greedily on a payment grid, and the uniform concentration will again imply that the best empirical payment on the grid will perform similarly to the best cost on the grid (up to a similar error of $\Olog\br*{S^2\beta_t/\sqrt{t}}$). All that remains is to relate the optimal cost on the chosen grid to the best continuous payment: we do so leveraging the one-sided Lipschitzness of the loss, as stated in \Cref{claim: lipschitzness}. In particular, we choose a sufficiently fine grid to ensure that the discretization error is negligible. Thus, the instantaneous error compared to the optimal loss $\ell^*$ is of order $\Olog\br*{S^2\beta_t/\sqrt{t}}$, and accumulating it across all rounds leads to the desired regret bound.
\end{proofsketch}

We now discuss two important aspects of our algorithm: its computational tractability and statistical efficiency (optimality of the regret bound).

\paragraph{Tractability of the loss minimization.} 
One caveat in the covariance correction term $\Sigma_n(c)-\Sigma_n(c_s)$ is that it might render the loss nonconvex. To mitigate this, we add a regularization term $\gamma_t\norm{\nu}^2$ to our loss. Then, we show in \Cref{prop: loss estimator known noise} that with high probability, calculating $\hat{\nu}_{t-1}(k)$ requires minimizing a strictly convex loss over a ball, which can be done efficiently. In particular, cases where the objective is nonconvex on some rounds are of low probability, so the learner can just choose arbitrary costs and predictions without affecting the performance. Nonetheless, it is important to note that the regularization is not strictly necessary for computational efficiency; the minimization of the unregularized empirical loss $\hat{L}^{kc}_t$ is in fact an instance of the trust region subproblem that can be efficiently solved even if the quadratic problem is nonconvex \citep[see, e.g., Section 8.2.7 in][]{beck2014introduction}. We also note that one could further increase the regularization term to make the objective $\Omega(R^2d+S^2)$-strongly convex, while only deteriorating the regret by constant factors, thereby greatly accelerating the loss minimization step.

\paragraph{Statistical efficiency.} At a glance, the regret rate that we obtained does not seem tight: the information sharing makes the problem similar to online linear regression, for which an $\Ocal\br*{d\ln T}$ regret is achievable \citep{vovk1997competitive}. Somewhat surprisingly, we show that this is not the case in our setting: even in one-dimensional problems, an $\Omega(\sqrt{T})$ is unavoidable (see proof in \Cref{appendix: known covariance lower bound}):
\begin{restatable}{theorem-rst}{lowerBoundKnown}
\label{theorem: known covariances lower bound}
    For any algorithm and $T\ge1$, there exists a one-dimensional instance with a known noise variance profile such that $\E\brs*{R(T)} \geq\frac{\sqrt{T}}{240}$.
\end{restatable}
The construction fixes the feature noise to be $\Ncal(0,1)$ for all payments $c<1/2$ and the noise to be deterministically zero when $c\ge 1/2$. We then create two instances with known $\theta^*=1$ and $x_t\sim \Ncal(0,1\pm \epsilon)$ and show that, depending on the instance, it is optimal to play either $c=0$ or $c=1/2$, thus reducing the problem to a two-armed bandit instance. From there, we use techniques from \citep{lattimore2020bandit} to derive the lower bound. We hypothesize that the $\sqrt{T}$ rate results from the discontinuity in the noise variance, and that for sufficiently smooth noise, a rate of $\Ocal(\ln T)$ might still be achievable, but leave this study for future work.

As a final remark, we note that the lower bound can be immediately generalized to $d$-dimensional features. In particular, we can choose a feature distribution that first samples a single coordinate (uniformly at random) to be non-zero, and then assigns its value according to the lower bound construction. This effectively creates $d$ independent instances, each running for approximately $T/d$ rounds. The regret of each instance will be lower bounded by $\Omega(\sqrt{T/d})$, so the overall regret from the interaction will be bounded by $\Omega(d\sqrt{T/d})=\Omega(\sqrt{Td})$.

\section{Extension to Unknown Noise Covariances}
\label{section: unknown noise covariance}

\begin{algorithm}[t]
\caption{Online Regression with Paid Features -- Unknown Noise Covariances} \label{alg: unknown noise covariance}

\begin{algorithmic}
\STATE {\bf Require:} $\delta\in(0,1), K\in\N$
% \STATE {\bf Set:} $c_k=\frac{k}{K}$ for $k\in\brc*{0,\dots,K}$, $ \hat{L}^{kc,R}_0(c,\nu)$
\FOR{$k=1,...,K$}
        \STATE  Pay $c_k\!=\!k/K$ and $\yhat_k\!=\!0$; observe $\xhat_k\!=\!\xhat_k(c_k)$ and $y_k$
        \STATE Set $\hat{L}^{uc}_K(k/K,\nu)=\br*{\xhat_k^{\top}\nu-y_k}^2+\lambda k/K$
    \ENDFOR
\FOR{$t=K+1,...,T$}
    \FOR{$k=1,...,K$}
        \STATE Calculate $\hat{\nu}_{t-1}(k)\in \argmin_{\nu:\norm*{\nu}\leq S}\hat{L}^{uc}_{t-1}(k/K,\nu)$
    \ENDFOR
    \STATE choose $k_{t-1}$ according to \cref{eq: UCB index unknown covariances}
    \STATE Pay $c_t=k_{t-1}/K$ and observe $\xhat_t = \xhat_t(c_t)$
    \STATE Predict $\yhat_t = \xhat_t^\top\hat{\nu}_{t-1}(k_{t-1})$ and observe $y_t$
    \STATE Update $\hat{L}^{uc}_{t}(k_{t-1}/K,\nu)$% via \cref{eq:loss regularized unknown covariances}
\ENDFOR
\end{algorithmic}
\end{algorithm}

We now tackle the more challenging case in which the noise covariances are not known in advance. In this situation, we unfortunately cannot share information between different costs as we did in the previous section, and resort to estimating the loss at a cost $c$ only based on samples collected with this specific cost. Formally, we define the loss 
\begin{align*}
    \hat{L}^{uc}_t(c,\nu) = \sum_{s=1}^t\br*{\br*{\xhat_s^{\top}\nu-y_s}^2 + \lambda c}\Ind{c_s=c},
\end{align*}
which fully overlaps with $\hat{L}^{kc}_t(c,\nu)$ when all samples are collected with $c_s=c$. Leveraging this insight, we utilize \Cref{prop: loss estimator known noise} to derive concentration results on $\hat{L}^{uc}_t(c,\nu)$ (see proof in \Cref{appendix: unknown covariances}). 
\begin{restatable}{corollary-rst}{lossEstUnknownNoise}
    \label{corollary: loss estimator unknown noise}
    Fix $c\in[0,1]$ and let $N_t(c) \!=\! \sum_{s=1}^t \Ind{c_s=c}$. Then, with probability at least $1-3\delta$, for all $t\ge1$ and $\nu\in\R^d$ s.t. $\norm*{\nu}\le S$, it holds that 
    {\small\begin{align*}
        &\abs*{\hat{L}^{uc}_t(c,\nu) - N_t(c)\ell(c,\nu)} 
        \leq  9S^2\sqrt{8N_t(c)\beta_t^2\ln\frac{3dt(t+1)}{\delta}},
    \end{align*}}
    where $\beta_t =\Olog(R^2d +S^2)$ is specified at \Cref{prop: loss estimator known noise}. %\Cref{lemma: matrix azuma for subgaussians}. 
\end{restatable}
We remark that without the information sharing term $\nu^{\top}\br*{\Sigma_n(c)-\Sigma_n(c_s)}\nu$, the loss $\hat{L}^{uc}_t(c,\nu)$ is always convex; therefore, the regularization is no longer needed and is omitted from the corollary.\footnote{The regularization could still be added to ensure strong convexity without affecting the results while requiring only minor algorithmic modifications.} 
A byproduct of separately estimating the loss for each cost is the need for exploration: as playing one cost does not yield sufficient information on all potential payments, we must test a diverse set of payments to identify the optimal cost $c^*$. We encourage our algorithm to explore by adding \emph{optimism} to $\hat{L}^{uc}_t$ \citep{auer2002using}, that is, replacing the empirical loss by the smallest plausible loss allowed by the confidence interval. Intuitively, the optimism introduces penalty terms that reduce the loss at payment levels that were not sufficiently observed, biasing toward their exploration. 

The full algorithm is depicted in \Cref{alg: unknown noise covariance}. We first limit the learner to a fixed grid of potential payments $c\in\brc*{k/K}_{k\in[K]}$; exploration is costly, so only these costs will be explored. We play each of these costs once as a loss initialization step. Then, at each round, the algorithm chooses the cost index that minimizes the optimistic average loss:
\begin{align}
    \label{eq: UCB index unknown covariances}
    k_{t-1}\in\argmin_{k\in[K]}&\left\{\frac{\hat{L}^{uc}_{t-1}\br*{k/K,\hat{\nu}_{t-1}(k)} }{N_{t-1}\br*{k/K}}\right.\nonumber\\
    &\quad \left. - 9S^2\sqrt{\frac{8\beta_t^2\ln\frac{3dt(t+1)}{\delta}}{N_{t-1}\br*{k/K}}}\right\},
\end{align}
where 
$$\hat{\nu}_{t-1}(k) \in \argmin_{\nu: \norm*{\nu}\leq S} \hat{L}^{uc}_{t-1}\br*{k/K,\nu}$$
is a linear predictor that minimizes the empirical loss. As before, after paying $c_t=k_{t-1}/K$, the learner observes $\xhat_t(c_t)$ and uses $\hat{\nu}_{t-1}(k_{t-1})$ to predict $\yhat_t=\xhat_t(c_t)^\top\hat{\nu}_{t-1}(k_{t-1})$. Afterwards, upon observing $y_t$, the learner only updates the loss at $c_t$ and continues to the next round.

The algorithm enjoys the following regret bound (see \Cref{appendix: unknown covariances} for the proof):
\begin{restatable}{theorem-rst}{regretUnknownNoiseCov}
    \label{theorem: regret unknown noise covariances}
    For any $T\ge1$, set $K=\ceil*{\frac{T^{1/3}\lambda^{2/3}}{\br*{S^2(R^2d+S^2)}^{2/3}}}$ and $\delta=\frac{1}{KT}$. Then, the regret of \Cref{alg: unknown noise covariance} is bounded by
\begin{align*}
    \Regret(T) = \Olog\br*{\br*{ S^2(R^2d+S^2)}^{2/3}\lambda^{1/3} T^{2/3}}.
\end{align*}
\end{restatable}
\begin{proofsketch}
The proof again relies on the loss concentration (\Cref{corollary: loss estimator unknown noise}), now combined with bandit techniques. 
In particular, we show that with high probability, the instantaneous regret of playing a cost $c_t$ compared to the best cost on the grid $\brc*{k/K}_{k\in[K]}$ is inversely proportional to the number of times it was previously played: of the order  $\Olog\br*{S^2\beta_{t-1}/\sqrt{N_{t-1}(c_t)}}$. On the other hand, working with a grid of $K$ costs leads to an instantaneous discretization error of $\lambda/K$ (due to the Lipschitzness proved in \Cref{claim: lipschitzness}). Therefore, the cumulative regret is bounded (w.h.p.) by
{\small\begin{align*}
    \Regret(T) &\lesssim \sum_{t=K+1}^T \br*{\frac{S^2\beta_T}{\sqrt{N_{t-1}(c_t)}}+\frac{\lambda}{K}}
    \approx S^2\beta_T\sqrt{KT} +\frac{\lambda T}{K}.
\end{align*}}
The last relation uses standard bandit arguments:  upon choosing a cost $c_t$, its corresponding count must increase, decreasing the denominator in future rounds in which this cost is played. Thus, the sum of the inverse counts cannot be too big. Fixing a grid size that optimally balances the two terms leads to the stated regret bound. 
\end{proofsketch}
\paragraph{Relation to Lipschitz Bandits.} The discretization approach on the costs and local loss estimation resembles existing approaches in the Lipschitz bandit literature (also called $\X$-armed bandits, \citealt{bubeck2011x}); there, a regret bound of $\Olog(T^{2/3})$ is obtained for one-dimensional problems. Indeed, although our problem is not Lipschitz, the one-sided Lipschitzness (\Cref{claim: lipschitzness}) might suffice for such a reduction. However, the resulting Lipschitz instance will be $d+1$-dimensional (representing both costs and predictions), thus leading to a regret bound of $\Olog(T^{\frac{d+2}{d+3}})$. In contrast, we show how to separately optimize over the linear predictors, so that we can perform the discretization only in one dimension (somewhat similar to \citealt{tullii2024improved}). This allows us to achieve an improved rate of $\Olog(T^{2/3})$.

Next, we discuss the tightness of \Cref{theorem: regret unknown noise covariances}. Our discretization-based bandit approach yields a worse regret compared to the $\Olog(\sqrt{T})$ rate with known noise covariances. This degradation is unavoidable, as we prove in the following one-dimensional lower bound:
\begin{restatable}{theorem-rst}{lowerBoundUnknown}
    \label{theorem: unknown covariances lower bound}
    For any algorithm and $T\ge1$, there exists a one-dimensional instance with an unknown noise variance profile such that $\E\brs*{R(T)} \geq\frac{T^{2/3}}{256}$.
\end{restatable}
The full proof can be found at \Cref{appendix: unknown covariances lower bound}.
\begin{proofsketch}
To derive this bound, we first show that for one dimensional problems with the parameters $\Sigma_x\triangleq\sigma_x^2=1$, $\theta^*=1$ and $\lambda=1/2$, the variance profile $\Sigma_n(c)\triangleq\sigma_n^2(c)=f(c)=\frac{1-c}{1+c}$ achieves a constant loss $\ell^*(c)=1/2$ for all $c\in[0,1]$. Then, we carefully
devise $K$ noise variance profiles (`instances') that smoothly deviate from $f(c)$ only at a small interval (`modified interval') of size $\approx 1/K$, with no interval overlaps between instances. In particular, the minimal loss inside these intervals is slightly lower than $1/2$ (by $\approx 1/K$). We remark that although our problem is very different from Lipschitz bandits, our construction draws inspiration from characteristics of hard Lipschitz bandit instances, and we indeed obtain the same bound as \citet{kleinberg2004nearly}. 

We then compare the behavior of any algorithm in these instances to that on a nominal instance with $\sigma_n^2(c)=f(c)$. By the pigeonhole principle, there must be an instance $k^*$ such that its modified interval was sampled less than $T/K$ times on average in the nominal instance. Since the nominal instance and $k^*$ are identical outside the interval (no information on the instance is gained), we show that for $K\approx T^{1/3}$, the algorithm cannot behave too differently on $k^*$; using information-theoretic tools \citep{garivier2019explore}, we show that it chooses costs in the modified intervals no more than $3T/4$ times. Since the regret for not choosing a cost in the modified interval is roughly $1/K$ (loss of $1/2$ instead of $1/2-1/K$), this implies a lower bound of 
\begin{align*}
    \Regret(T)\gtrsim \br*{T-\frac{3T}{4}}\frac{1}{K}=\Omega\br*{T^{2/3}}.
\end{align*}
\end{proofsketch}

\section{Summary and Future Work}
\label{section: summary}

In this work, we studied online linear regression problems with stochastic features that are corrupted by noise, and the learner can pay to reduce this noise. This setting is applicable, for example, in settings where features are measured in experiments, problems with privacy-protected information, and more. 
We focused on two variants of this problem: in the first, the learner has full knowledge of the noise covariance across all payment levels (e.g., privacy noise with known characteristics); in the second, the learner must estimate the effect of the payment on the quality of the observed features. In both cases, we devised learner algorithms that achieve order-optimal regret bounds as a function of the interaction length (up to logarithmic factors).

While the algorithms presented in this work achieve order-optimal horizon-dependence in their regret bounds, it is unclear whether these algorithms are optimal with respect to other problem parameters. In particular, all our lower bounds are derived on one-dimensional instances, and so the optimal dimension dependence remains to be determined. In addition, while both algorithms are polynomial, each time step requires solving a convex optimization problem over a grid of costs. It would be beneficial to derive more computationally efficient algorithms, including algorithms that build adaptive grids or avoid discretizing the payments.

Our work could also be extended beyond linear regression. Specifically, it would be interesting to extend our results to non-linear regression problems $y_t=f(x_t)$ (for possibly nonlinear $f\in\F$), also potentially deviating from the quadratic loss. In this context, it is especially relevant to study \emph{agnostic} settings: limiting the algorithm to a smaller class of predictors that does not include the mapping that minimizes the loss between $\yhat(x_t)$ and $y_t$. This might be necessary since the optimal predictors could be extremely complicated, as we discussed in \Cref{remark: linear}. Specifically, we limited our predictors to be linear, similarly to the mapping that created the original outputs $y_t$, but for more complex regression models, it might be beneficial to tailor the prediction class $\Pcal$ to be different than the true regression model class $\F$.

Finally, in some situations, we need to pay for each feature separately and/or can obtain the same feature from multiple sources (for different payments and with different noise). It is then interesting to determine the optimal (combinatorial) payment scheme and devise algorithms that learn to optimally aggregate features gathered from different sources.

\section*{Acknowledgments}
% This research
NM was supported by Israel Science Foundation research grant (ISF’s No. 4118/25) and the
Maimonides Fund’s Future Scientists Center.
NCB acknowledges the financial support from the MUR PRIN grant 2022EKNE5K (Learning in Markets and Society), the EU Horizon CL4-2021-HUMAN-01 research and innovation action under grant agreement 101070617, project ELSA, and the FAIR (Future Artificial Intelligence Research) project, funded by the NextGenerationEU program within the PNRR-PE-AI scheme.

\bibliography{aaai2026}
\clearpage

\newpage

\onecolumn
\appendix
\setcounter{secnumdepth}{2}
\renewcommand{\thesubsection}
  {\Alph{section}.\arabic{subsection}}

\section{Properties of the Problem}
\label{appendix: properties}

\optimalLoss*
\begin{proof}
    By definition, it holds that 
    \begin{align*}
    \ell(c,\nu) &= \E_{x\sim D_x, n\sim D_n(c)}\brs*{\br*{\br*{x+n}^\top\nu - x^\top\theta^*}^2} + \lambda c \\
    & = \E_{x\sim D_x, n\sim D_n(c)}\brs*{\br*{n^\top\nu + x^\top\br*{\nu-\theta^*}}^2} +\lambda c \\
    & \overset{(1)}{=} \E_{x\sim D_x, n\sim D_n(c)}\brs*{\nu^\top n n^\top\nu +  \br*{\nu-\theta^*}^\top xx^\top\br*{\nu-\theta^*}} +\lambda c \\
    & \overset{(2)}{=}  \nu^\top \Sigma_n(c)\nu + \br*{\nu-\theta^*}^\top \Sigma_x \br*{\nu-\theta^*} +\lambda c\\
    & = \nu^\top \Sigma_n(c)\nu + \nu^\top \Sigma_x\nu - 2\nu^\top \Sigma_x\theta^* + {\theta^*}^\top \Sigma_x{\theta^*} +\lambda c \\
    & = \nu^\top \Sigma_{\xhat}(c)\nu - 2\nu^\top \Sigma_x\theta^* + {\theta^*}^\top \Sigma_x{\theta^*} +\lambda c .
\end{align*}
Relation $(1)$ holds since $x$ and $n(c)$ are independent and $n(c)$ has zero mean, and relation $(2)$ holds by the definition of the covariances. The function is strongly convex w.r.t. $\nu$ (due to the regularity of the covariances), so there exists a single optimal solution, obtainable by the first-order optimality condition:
\begin{align*}
    2\Sigma_{\xhat}(c)\nu - 2\Sigma_x\theta^*=0.
\end{align*}
Reorganizing, we get that $\bar{\nu}(c) = \br*{\Sigma_{\xhat}(c)}^{-1}\Sigma_x\theta^*$ is the unique unconstrained global minimizer. In particular, if it satisfies the constraint $\norm*{\bar{\nu}(c)}\leq S$, it is also the constrained minimizer, and thus $\nu^*(c) = \br*{\Sigma_{\xhat}(c)}^{-1}\Sigma_x\theta^*$. Substituting back to $\ell(c,\nu)$ directly leads to the stated optimal loss:
    \begin{align*}
    \ell^*(c) & = \nu^*(c)^\top \Sigma_{\xhat}(c)\nu^*(c) - 2\nu^*(c)^\top \Sigma_x\theta^* + {\theta^*}^\top \Sigma_x{\theta^*} +\lambda c \\
    & = {\theta^*}^\top \Sigma_x\br*{\Sigma_{\xhat}(c)}^{-1} \Sigma_{\xhat}(c)\br*{\Sigma_{\xhat}(c)}^{-1}\Sigma_x\theta^* - 2{\theta^*}^\top \Sigma_x\br*{\Sigma_{\xhat}(c)}^{-1} \Sigma_x\theta^* + {\theta^*}^\top \Sigma_x{\theta^*} +\lambda c\\
    & = {\theta^*}^\top \Sigma_x{\theta^*} - {\theta^*}^\top \Sigma_x\br*{\Sigma_{\xhat}(c)}^{-1} \Sigma_x\theta^* + \lambda c.
\end{align*}
Note that $\Sigma_{\xhat}(c)$ is invertible by the assumption $\Sigma_{\xhat}(1)\succ0$ and the monotonicity of the cost, so both expressions for $\nu^*(c)$ and $\ell^*(c)$ are well-defined. 
\end{proof}

\begin{claim}
    \label{claim: one-dimension loss}
    Assume $d=1$. Then, for $\bar{\nu}(c)=\br*{\Sigma_{\xhat}(c)}^{-1}\Sigma_x\theta^*$, it holds that $\abs*{\bar{\nu}(c)}\leq S$ for all $c\in[0,1]$.
\end{claim}
\begin{proof}
    In dimension one, all parameters $\Sigma_x,\Sigma_n(c)$ and $\theta^*$ are scalars, and by the assumptions, $\Sigma_x,\Sigma_n(c)\ge0$, $\Sigma_x+\Sigma_n(c)>0$ and $\abs*{\theta^*}\leq S$. Given this, the claim immediately holds from substitution:
    \begin{align*}
        \abs*{\bar{\nu}(c)}
        = \abs*{\br*{\Sigma_{\xhat}(c)}^{-1}\Sigma_x\theta^*}
        = \underbrace{\abs*{\frac{\Sigma_x}{\Sigma_x+\Sigma_n(c)}}}_{\leq 1}\underbrace{\abs*{\theta^*}}_{\leq S} \leq S.
    \end{align*}
\end{proof}

\lossProperties*
\begin{proof}
    By the noise covariance monotonicity assumption, we have that
        \begin{align*}
            \Sigma_{\xhat}(c_2)
            =\Sigma_x +\Sigma_n(c_2) 
            \preceq \Sigma_x +\Sigma_n(c_1) 
            =\Sigma_{\xhat}(c_1),
        \end{align*}
        and thus, for any $\nu$ and $0\leq c_1\leq c_2\leq 1$, it holds by \Cref{claim: opt values} that
        \begin{align*}
            \ell(c_1,\nu)
            &= \nu^\top \Sigma_{\xhat}(c_1)\nu - 2\nu^\top \Sigma_x\theta^* + {\theta^*}^\top \Sigma_x{\theta^*} +\lambda c_1 \\
            &\geq \nu^\top \Sigma_{\xhat}(c_2)\nu - 2\nu^\top \Sigma_x\theta^* + {\theta^*}^\top \Sigma_x{\theta^*} +\lambda c_1 \tag{$\Sigma_{\xhat}(c_1)\succeq \Sigma_{\xhat}(c_2)$} \\
            &= \nu^\top \Sigma_{\xhat}(c_2)\nu - 2\nu^\top \Sigma_x\theta^* + {\theta^*}^\top \Sigma_x{\theta^*} +\lambda c_2 +\lambda (c_1-c_2) \\
            & = \ell(c_2,\nu)+\lambda (c_1-c_2).
        \end{align*}
        Minimizing over $\nu$ of norm smaller than $S$ in both sides of the inequality, we get 
        \begin{align*}
            \ell^*(c_1) \geq \ell^*(c_2)+\lambda (c_1-c_2),
        \end{align*}
        and reorganizing this inequality concludes the proof.
\end{proof}

\regretDecomp*
\begin{proof}
    Recall that  $c_t$ and $\nu_t$ are assumed to be completely determined by $I^o_{t-1}$ and that $x_t,n_t$ are drawn respectively from $D_x,D_n(c_t)$ given $I_{t-1}$ (and therefore, also given $I^o_{t-1}$). Then, for any $t$, we can write
    \begin{align}
        \E\brs*{\ell(c_t,\nu_t)}
        &= \E\brs*{\E_{x\sim D_x, n\sim D_n(c_t)}\brs*{\br*{\br*{x+n}^\top\nu_t - x^\top\theta^*}^2 + \lambda c_t \vert c_t,\nu_t}} \nonumber\\
        &= \E\brs*{\E_{x\sim D_x, n\sim D_n(c_t)}\brs*{\br*{\br*{x+n}^\top\nu_t - x^\top\theta^*}^2 + \lambda c_t \vert I^o_{t-1}}} \nonumber\\
        & =\E\brs*{\E\brs*{\br*{\br*{x_t+n_t}^\top\nu_t - x_t^\top\theta^*}^2 + \lambda c_t \vert I^o_{t-1}}} \nonumber\\
        & = \E\brs*{\br*{\xhat_t^\top\nu_t - x_t^\top\theta^*}^2 + \lambda c_t }. \label{eq: loss from ind to process}
    \end{align}
    
    Starting from the regret definition and using this identity, we decompose it as follows.
    \begin{align*}
        \Regret(T) 
        &= \E\brs*{\sum_{t=1}^T \br*{\br*{\yhat_t(\xhat_t) - y_t}^2 + \lambda c_t}} - T\ell^* \\
        & = \E\brs*{\sum_{t=1}^T \br*{\br*{\yhat_t(\xhat_t) - y_t}^2 + \lambda c_t - \ell(c_t,\nu_t)}} + \E\brs*{\sum_{t=1}^T \ell(c_t,\nu_t)}- T\ell^*\\
        & = \E\brs*{\sum_{t=1}^T \br*{\br*{\yhat_t(\xhat_t) - y_t}^2 + \lambda c_t - \br*{\br*{\xhat_t^\top\nu_t - y_t}^2 + \lambda c_t}}}   + \E\brs*{\sum_{t=1}^T \ell(c_t,\nu_t)}- T\ell^*\tag{by \cref{eq: loss from ind to process}}\\
        & = \E\brs*{\sum_{t=1}^T \br*{\br*{\yhat_t(\xhat_t) - y_t}^2 - \br*{\br*{\xhat_t^\top\nu_t - y_t}^2}}}  + \E\brs*{\sum_{t=1}^T \br*{\ell(c_t,\nu_t)- \ell^* }}.
    \end{align*}
    Specifically, if $\yhat_t(\xhat_t) = \xhat_t^\top\nu_t$, then the first term is canceled out and we get the desired expression.
\end{proof}

\begin{claim}
    \label{claim: gaussian loss}
    Assume that $D_x$ and $D_n(c)$ are zero-mean Gaussian independent distributions and assume that for $\bar{\nu}(c)=\br*{\Sigma_{x}+\Sigma_n(c)}^{-1}\Sigma_x\theta^*$, it holds that $\norm*{\bar{\nu}(c)}\leq S$ for all $c\in[0,1]$. Then, 
    \begin{align*}
        \Regret(T) 
        & \geq \E\brs*{\sum_{t=1}^T \br*{\ell^*(c_t) - \ell^*(c^*)}}.
    \end{align*}
\end{claim}
\begin{proof}
    Denote the history up to time $t$ (up to $c_t$ but not $x_t$ or $n_t$) by $I_{t-1}$. We use the fact that the mean squared error is always minimized by the conditional expectation; that is, for any two random variables $Z\in\R^d, W\in\R$ and any function $f$, it holds that
    \begin{align*}
        \E\brs*{\br*{f(Z)-W}^2} \geq \E\brs*{\br*{\E\brs*{W\vert Z} - W}^2}.
    \end{align*}
    In particular, for $Z=\xhat_t$ and $W=y_t=x_t^\top\theta^*$, it holds that
    \begin{align*}
        \E\brs*{\br*{\yhat_t(\xhat_t) - y_t}^2\vert I_{t-1}}
        &\geq \E\brs*{\br*{\E\brs*{x_t^\top\theta^* \vert \xhat_t,I_{t-1}} - x_t^\top\theta^*}^2\vert I_{t-1}}
        = \E\brs*{\br*{\br*{\E\brs*{x_t \vert \xhat_t, I_{t-1}} - x_t}^\top\theta^*}^2\vert I_{t-1}} \\
        &={\theta^*}^\top \E\brs*{\br*{\E\brs*{x_t\vert \xhat_t,I_{t-1}}-x_t}\br*{\E\brs*{x_t\vert \xhat_t,I_{t-1}}-x_t}^\top\vert I_{t-1} }\theta^*.
    \end{align*}
    Conditioned on the history, the vector $z_t=\begin{pmatrix} x_t \\ \xhat_t \end{pmatrix}$ is a Gaussian random vector of mean $\begin{pmatrix} 0 \\ 0 \end{pmatrix}$ and covariance 
    $\begin{pmatrix} \Sigma_x & \Sigma_x \\ \Sigma_x & \Sigma_{\xhat}(c_t) \end{pmatrix}$ (due to the independence of $x_t$ and $n_t$). Thus, by the classic results on the conditional expectation of Gaussian vectors, we have
    \begin{align*}
        \E\brs*{x_t\vert \xhat_t,I_{t-1}} 
        = \Sigma_x\br*{\Sigma_x+\Sigma_n(c_t)}^{-1}\xhat_t.
    \end{align*}
    Substituting back, we get
    \begin{align*}
        \E\brs*{\br*{\yhat_t(\xhat_t) - y_t}^2\vert I_{t-1}}
        &\geq {\theta^*}^\top \E\brs*{\br*{\Sigma_x\br*{\Sigma_x+\Sigma_n(c_t)}^{-1}\xhat_t-x_t}\br*{\Sigma_x\br*{\Sigma_x+\Sigma_n(c_t)}^{-1}\xhat_t-x_t}^\top \vert I_{t-1}}\theta^* \\
        & = {\theta^*}^\top\Sigma_x\br*{\Sigma_x+\Sigma_n(c_t)}^{-1} \underbrace{\E\brs*{\xhat_t\xhat_t^\top \vert I_{t-1} }}_{=\Sigma_x+\Sigma_n(c_t)}\br*{\Sigma_x+\Sigma_n(c_t)}^{-1}\Sigma_x\theta^* 
        - {\theta^*}^\top\Sigma_x\br*{\Sigma_x+\Sigma_n(c_t)}^{-1}\underbrace{\E\brs*{\xhat_t x_t^\top \vert I_{t-1}}}_{=\Sigma_x}\theta^* \\
        & \quad - {\theta^*}^\top \underbrace{\E\brs*{x_t\xhat_t\hat\vert I_{t-1}}}_{=\Sigma_x}\br*{\Sigma_x+\Sigma_n(c_t)}^{-1}\Sigma_x\theta^* 
        + {\theta^*}^\top \underbrace{\E\brs*{x_tx_t^\top\vert I_{t-1}}}_{=\Sigma_x}\theta^* \\
        & = {\theta^*}^\top \Sigma_x\theta^* - {\theta^*}^\top \Sigma_x\br*{\Sigma_x+\Sigma_n(c_t)}^{-1}\Sigma_x\theta^* \\
        & = \ell^*(c_t)-\lambda c_t,
    \end{align*}
    where the last equality is by \Cref{claim: opt values}. Substituting into the regret, we get 
    \begin{align*}
    \Regret(T) 
    &= \E\brs*{\sum_{t=1}^T \br*{\br*{\yhat_t(\xhat_t) - y_t}^2 + \lambda c_t}} - T\ell^* \\
    & = \E\brs*{\sum_{t=1}^T \br*{\E\brs*{\br*{\yhat_t(\xhat_t) - y_t}^2\vert I_{t-1}} + \lambda c_t}} - T\ell^* \\
    & \geq \E\brs*{\sum_{t=1}^T \br*{\br*{\ell^*(c_t)-\lambda c_t} + \lambda c_t}} - T\ell^* \\
    & = \E\brs*{\sum_{t=1}^T \br*{\ell^*(c_t) - \ell^*}}.
\end{align*}
\end{proof}

\begin{claim}
    \label{claim:max loss}
    Assume that $x,n\in\R^d$ are $R^2$-subgaussian independent vectors s.t. $E[x]=\bar{x}$ and $\E[n]=0$, where $\norm*{\bar{x}}\leq S$. Then,
    \begin{align*}
        \max_{c\in[0,1],\nu:\norm{\nu}\leq S}\brc*{\E\brs*{\br*{\br*{x+n}^\top\nu - x^\top\theta^*}^2} + \lambda c} \leq 6S^2\br*{R^2d+S^2}+\lambda.
    \end{align*}
\end{claim}
\begin{proof}
    For any $c\in[0,1]$ and $\nu\in\R^d$ s.t. $\norm*{\nu}\le S$, it holds that
    \begin{align*}
        \E\brs*{\br*{\br*{x+n}^\top\nu - x^\top\theta^*}^2} + \lambda c
        & \leq 2\E\brs*{\br*{\br*{x+n}^\top\nu}^2} + 2\E\brs*{\br*{x^\top\theta^*}^2}  +\lambda \tag{$(a+b)^2\leq 2a^2+2b^2$}\\
        &\leq 2\E\brs*{S^2\norm{x+n}^2} + 2\E\brs*{S^2\norm{x}^2}  +\lambda \tag{Cauchy Schwartz}
    \end{align*}
    We now use the fact that $x$ is $R^2$ subgaussian and $x+n$ is $2R^2$ subgaussian, both of expectation $\bar{x}$. Thus, by \Cref{lemma: subgaussian norm}, and using the identity $\E\brs*{\norm*{Y-\E[Y]}^2}=\E\brs*{\norm*{Y}^2}-\norm*{\E[Y]}^2$ (which holds for any random vector), we get
    \begin{align*}
        \E\brs*{\br*{\br*{x+n}^\top\nu - x^\top\theta^*}^2} + \lambda c
        &\leq 2S^2\br*{\E\brs*{\norm{x+n-\bar{x}}^2}+\norm{\bar{x}}^2} + 2S^2\br*{\E\brs*{\norm{x-\bar{x}}^2}+\norm{\bar{x}}^2}  +\lambda \\
        & \leq 2S^2\br*{2R^2d+S^2} + 2S^2\br*{R^2d+S^2} + \lambda\tag{\Cref{lemma: subgaussian norm}}\\
        &  \leq 6S^2\br*{R^2d+S^2} + \lambda.
    \end{align*}
\end{proof}

\clearpage

\section{Proofs for Known Noise Covariances}
\label{appendix: known covariances}
\lossEstKnownNoise*
\begin{proof}
The  loss estimator can be written as follows:
\begin{align}
    \hat{L}^{kc}_t(c,\nu) 
    &= \sum_{s=1}^t\br*{\nu^{\top}\br*{\xhat_s\xhat_s^\top + \Sigma_n(c)-\Sigma_n(c_s)}\nu - 2\nu^\top\xhat_sy_s + y_s^2+\lambda c}\nonumber \\
    & = \nu^{\top}\br*{\sum_{s=1}^t\br*{\xhat_s\xhat_s^\top + \Sigma_n(c)-\Sigma_n(c_s)}}\nu - 2\nu^{\top}\br*{\sum_{s=1}^t\xhat_s x_s^{\top}}\theta^* + {\theta^*}^{\top}\br*{\sum_{s=1}^tx_s^{\top}x_s}\theta^* + t\lambda c \label{eq: loss known covariance manip 1}.
\end{align}
Therefore, by \Cref{claim: opt values}, for any cost $c$ and parameter $\nu$, we have 
\begin{align*}
    \hat{L}^{kc}_t(c,\nu) &- t\ell(c,\nu)\\
    &= \nu^{\top}\br*{\sum_{s=1}^t\br*{\xhat_s\xhat_s^\top - \Sigma_x-\Sigma_n(c_s)}}\nu - 2\nu^{\top}\br*{\sum_{s=1}^t\xhat_s x_s^{\top}-\Sigma_x}\theta^* + {\theta^*}^{\top}\br*{\sum_{s=1}^tx_s x_s^{\top} - \Sigma_x}\theta^* \\
    & = \nu^{\top}\br*{\sum_{s=1}^t\br*{\xhat_s\xhat_s^\top - \Sigma_x-\Sigma_n(c_s)}}\nu - 2\nu^{\top}\br*{\sum_{s=1}^tx_s x_s^{\top}-\Sigma_x}\theta^* - 2\nu^{\top}\br*{\sum_{s=1}^tn_s x_s^{\top}}\theta^* + {\theta^*}^{\top}\br*{\sum_{s=1}^tx_s x_s^{\top} - \Sigma_x}\theta^*.
\end{align*}
All but a single term form outer products. The remaining mixed term can also be represented as
\begin{align*}
    \nu^{\top}\br*{\sum_{s=1}^tn_s x_s^{\top}}\theta^*
    = \begin{pmatrix} 0^\top & \nu^\top \end{pmatrix}\sum_{s=1}^t \br*{\begin{pmatrix} x_s \\ n_s \end{pmatrix}\begin{pmatrix} x_s \\ n_s \end{pmatrix}^\top}\begin{pmatrix} \theta^* \\ 0 \end{pmatrix}
    =  \begin{pmatrix} 0^\top & \nu^\top \end{pmatrix}\sum_{s=1}^t \br*{\begin{pmatrix} x_s \\ n_s \end{pmatrix}\begin{pmatrix} x_s \\ n_s \end{pmatrix}^\top - \begin{pmatrix} \Sigma_x & 0 \\ 0 & \Sigma_n(c_s) \end{pmatrix}}\begin{pmatrix} \theta^* \\ 0 \end{pmatrix},
\end{align*}
and using the inequality $x^\top A y \leq \norm*{x}\norm{y}\norm{A}_{op}$ for any $x,y\in\R^d$ and $A\in\R^{d\times d}$, alongside the bounds $\norm{\nu},\norm*{\theta^*}\leq S$, we get
\begin{align*}
    \abs*{\hat{L}^{kc}_t(c,\nu) -t\ell(c,\nu)}
    &\leq S^2 \norm*{\sum_{s=1}^t\br*{\xhat_s\xhat_s^\top - \Sigma_x-\Sigma_n(c_s)}}_{op} + 3S^2  \norm*{\sum_{s=1}^t\br*{x_sx_s^\top - \Sigma_x}}_{op} \\
    &\quad + 2S^2 \norm*{\sum_{s=1}^t \br*{\begin{pmatrix} x_s \\ n_s \end{pmatrix}\begin{pmatrix} x_s \\ n_s \end{pmatrix}^\top - \begin{pmatrix} \Sigma_x & 0 \\ 0 & \Sigma_n(c_s) \end{pmatrix}}}_{op}.
\end{align*}
Now, notice that $x_t$ and $n_t$ are $R^2$ conditionally subgaussian and conditionally independent; hence, $\xhat_t=x_t+n_t$ is $2R^2$ conditionally subgaussian and $z_t = \begin{pmatrix} x_t \\ n_t \end{pmatrix}$ is $R^2$ conditionally subgaussian  (w.r.t. the filtration $\F_t=\sigma(I_t)$), and both their expectations are of norm smaller than $S$. We can therefore apply \Cref{lemma: matrix azuma for subgaussians} on each of the three norms (noticing that the conditional covariance of each term is as needed for the lemma). While doing so, we remark that replacing $R^2\rightarrow 2R^2$ or $d\rightarrow 2d$ can increase $\beta_t$ by at most a factor of $2$. Following this application of  \Cref{lemma: matrix azuma for subgaussians}, we get that w.p. at least $1-3\delta$, for all $t\ge1$:
\begin{align*}
    &\norm*{\sum_{s=1}^t\br*{\xhat_s\xhat_s^\top - \Sigma_x-\Sigma_n(c_s)}}_{op} < 2\sqrt{8t\beta_t^2\ln\frac{3dt(t+1)}{\delta}},\\
    &\norm*{\sum_{s=1}^t\br*{x_sx_s^\top - \Sigma_x}}_{op} < \sqrt{8t\beta_t^2\ln\frac{3dt(t+1)}{\delta}},\qquad \textrm{and}\\
    &\norm*{\sum_{s=1}^t \br*{\begin{pmatrix} x_t \\ n_t \end{pmatrix}\begin{pmatrix} x_t \\ n_t \end{pmatrix}^\top - \begin{pmatrix} \Sigma_x & 0 \\ 0 & \Sigma_n(c) \end{pmatrix}}}_{op} < 2\sqrt{8t\beta_t^2\ln\frac{3dt(t+1)}{\delta}}.
\end{align*}
In particular, for all $t\ge1$, $c\in[0,1]$ and $\nu\in\R^d$ s.t. $\norm*{\nu}\le S$, it holds that 
\begin{align*}
    \abs*{\hat{L}^{kc}_t(c,\nu) -t\ell(c,\nu)}
    &\leq 2S^2 \sqrt{8t\beta_t^2\ln\frac{3dt(t+1)}{\delta}} + 3S^2\sqrt{8t\beta_t^2\ln\frac{3dt(t+1)}{\delta}} + 4S^2\sqrt{8t\beta_t^2\ln\frac{3dt(t+1)}{\delta}} \\
    &= 9S^2\sqrt{8t\beta_t^2\ln\frac{3dt(t+1)}{\delta}}.
\end{align*}
Finally, we move our focus to the regularized loss $\hat{L}^{kc,R}_t(c,\nu)$, which similarly to \cref{eq: loss known covariance manip 1}, can be written as 
\begin{align*}
    \hat{L}^{kc,R}_t(c,\nu) = \hat{L}^{kc}_t(c,\nu) + \gamma_t\norm*{\nu}^2
    =\nu^{\top}\sum_{s=1}^t\br*{\xhat_s\xhat_s^\top + \Sigma_n(c)-\Sigma_n(c_s)+\gamma_t I}\nu - 2\nu^\top\sum_{s=1}^t\br*{\xhat_sy_s} + \sum_{s=1}^ty_s^2 + t\lambda c
\end{align*}
The loss is clearly quadratic in $\nu$ and is strictly convex iff the matrix in the quadratic form is positive definite. Indeed, under the aforementioned event, we have 
\begin{align*}
    \lambda_{\min}\br*{ \sum_{s=1}^t\br*{\xhat_s\xhat_s^\top + \Sigma_n(c)-\Sigma_n(c_s)+\gamma_t I}}
    &= \gamma_t +  \lambda_{\min}\br*{ \sum_{s=1}^t\br*{\xhat_s\xhat_s^\top-\Sigma_n(c_s) - \Sigma_x} + t\br*{\Sigma_x + \Sigma_n(c)}}\\
    &\geq \gamma_t +  \lambda_{\min}\br*{ \sum_{s=1}^t\br*{\xhat_s\xhat_s^\top-\Sigma_n(c_s) - \Sigma_x}} \tag{$\Sigma_x,\Sigma_n(c)\succeq0$}\\
    &\geq \gamma_t -  \norm*{ \sum_{s=1}^t\br*{\xhat_s\xhat_s^\top-\Sigma_n(c_s) - \Sigma_x}}_{op}\\
    &>0. 
\end{align*}
\end{proof}
\clearpage

\regretKnownNoiseCov*
\begin{proof}
    Define the good event $\Gcal$ under which all the results of \Cref{prop: loss estimator known noise} hold; by the proposition, we know that the probability of $\Gcal$ is at least $1-3/T$. In particular, when $\Gcal$ holds, for all $t\ge1$, $c\in[0,1]$ and $\nu$ s.t. $\norm{\nu}\leq S$, we have that
    \begin{align}
        \label{eq: good event known covariances}
        \abs*{\hat{L}^{kc}_t(c,\nu) - t\ell(c,\nu)} \leq  9S^2\sqrt{8t\beta_t^2\ln\frac{3dt(t+1)}{\delta}}.
    \end{align}
    We also define the regularized expected loss by $\ell^R_t(c,\nu) = \ell(c,\nu) + \frac{1}{t}\gamma_t\norm{\nu}^2$ and its minimizer $\nu^R_t(c) \in \argmin_{\nu: \norm{\nu}\leq S}\ell^R_t(c,\nu)$. Similarly, define the optimal cost on the grid by $k^R_t\in\argmin_{k\in[K]} \ell^R_t(k/K,\nu^R_t(k/K))$.
    Then, under $G$, for all $t\ge1$, we have
    \begin{align*}
        \ell\br*{\frac{k_{t}}{K},\hat{\nu}_{t}(k_{t})}
        & \leq  \ell\br*{\frac{k_{t}}{K},\hat{\nu}_{t}(k_{t})} + \frac{1}{t}\gamma_t\norm{\hat{\nu}_{t}(k_{t})}^2 \tag{The regularization is nonnegative}\\
        & \leq \frac{1}{t}\br*{\hat{L}^{kc}_t\br*{\frac{k_{t}}{K},\hat{\nu}_{t}(k_{t})}+\gamma_t\norm{\hat{\nu}_{t}(k_{t})}^2} + 9S^2\sqrt{\frac{8\beta_t^2\ln\frac{3dt(t+1)}{\delta}}{t}} \tag{By \cref{eq: good event known covariances} }\\
        & \overset{(1)}{=} \frac{1}{t}\min_{k\in[K],\nu:\norm{\nu}\leq S}\brc*{\hat{L}^{kc}_t\br*{\frac{k}{K},\nu}+\gamma_t\norm{\nu}^2} + 9S^2\sqrt{\frac{8\beta_t^2\ln\frac{3dt(t+1)}{\delta}}{t}} \\
        & \leq  \frac{1}{t}\br*{\hat{L}^{kc}_t\br*{\frac{k^R_t}{K},\nu^R_t\br*{\frac{k^R_t}{K}}}+\gamma_t\norm*{\nu^R_t\br*{\frac{k^R_t}{K}}}^2} + 9S^2\sqrt{\frac{8\beta_t^2\ln\frac{3dt(t+1)}{\delta}}{t}} \\
        & \leq  \br*{\ell\br*{\frac{k^R_t}{K},\nu^R_t\br*{\frac{k^R_t}{K}}}+\frac{\gamma_t}{t}\norm*{\nu^R_t\br*{\frac{k^R_t}{K}}}^2} + 18S^2\sqrt{\frac{8\beta_t^2\ln\frac{3dt(t+1)}{\delta}}{t}}  \tag{By \cref{eq: good event known covariances} } \\
        & \overset{(2)}= \min_{k\in[K],\nu:\norm{\nu}\leq S}\brc*{\ell\br*{\frac{k}{K},\nu} + \frac{\gamma_t}{t}\norm*{\nu}^2} + 18S^2\sqrt{\frac{8\beta_t^2\ln\frac{3dt(t+1)}{\delta}}{t}} \\
        & \leq \min_{k\in[K],\nu:\norm{\nu}\leq S}\brc*{\ell\br*{\frac{k}{K},\nu}} + \frac{S^2\gamma_t}{t} + 18S^2\sqrt{\frac{8\beta_t^2\ln\frac{3dt(t+1)}{\delta}}{t}} \tag{$\norm*{\nu^R_t\br*{\frac{k^R_t}{K}}}\leq S$} \\
        & \overset{(3)}= \min_{k\in[K]}\brc*{\ell^*\br*{\frac{k}{K}}}+ 20S^2\sqrt{\frac{8\beta_t^2\ln\frac{3dt(t+1)}{\delta}}{t}},  
    \end{align*}
    where relation $(1)$ is by the optimality of $k_{t}$ and $\hat{\nu}_{t}$ w.r.t. the empirical regularized loss and $(2)$ is by the optimality of $k^R_t$ and $\nu^R_t\br*{\frac{k^R_t}{K}}$ w.r.t. the true regularized loss. Relation $(3)$ substitutes $\gamma_t$ and $\ell^*$.
    
    We continue by analyzing the relation between the optimal loss and the loss over the grid:
    \begin{align}
        \label{eq: global to discretized minimizer}
        \ell^*
        = \min_{c\in[0,1]}\ell^*(c)
        =  \min_{k\in[K]}\min_{c\in\left[\frac{k-1}{K},\frac{k}{K}\right)}\ell^*(c)
        \overset{(*)}{\geq}  \min_{k\in[K]}\min_{c\in\left[\frac{k-1}{K},\frac{k}{K}\right)}\brc*{\ell^*\br*{\frac{k}{K}}-\lambda\br*{\frac{k}{K}-c}}
        = \min_{k\in[K]}\ell^*\br*{\frac{k}{K}} - \frac{\lambda}{K},
    \end{align}
    where relation $(*)$ holds due to the one-sided Lipschitzness on loss (see \Cref{claim: lipschitzness}). Combining both inequalities, we get 
    \begin{align*}
        \ell\br*{\frac{k_{t}}{K},\hat{\nu}_{t}(k_{t})}
        \leq \ell^* + 20S^2\sqrt{\frac{8\beta_t^2\ln\frac{3dt(t+1)}{\delta}}{t}}  +\frac{\lambda}{K}.
    \end{align*}
    We are now ready to bound the regret, using the regret decomposition of \Cref{lemma: regret decomposition}:
    \begin{align*}
        \Regret(T) 
        & = \E\brs*{\sum_{t=1}^T \br*{\ell(c_t,\nu_t) - \ell^*}} \\
        & = \E\brs*{\sum_{t=1}^T \br*{\ell\br*{\frac{k_{t-1}}{K},\hat{\nu}_{t-1}(k_{t-1})} - \ell^*}} \\
        & \leq \E\brs*{\sum_{t=2}^T \br*{\ell\br*{\frac{k_{t-1}}{K},\hat{\nu}_{t-1}(k_{t-1})} - \ell^*}} + \max_{c\in[0,1],\nu:\norm{\nu}\leq S}\ell(c,\nu) \tag{separating $T=1$} \\
        & \leq  \E\brs*{\Ind{\Gcal}\sum_{t=2}^T \br*{\ell\br*{\frac{k_{t-1}}{K},\hat{\nu}_{t-1}(k_{t-1})} - \ell^*}} + \underbrace{\Pr\brc*{\bar{\Gcal}}T}_{\leq 3}\max_{c\in[0,1],\nu:\norm{\nu}\leq S}\ell(c,\nu) + \max_{c\in[0,1],\nu:\norm{\nu}\leq S}\ell(c,\nu) \\
        & \leq \sum_{t=1}^{T-1}\br*{20S^2\sqrt{\frac{8\beta_t^2\ln\frac{3dt(t+1)}{\delta}}{t}}  +\frac{\lambda}{K}} + 4\max_{c\in[0,1],\nu:\norm{\nu}\leq S}\ell(c,\nu) \\
        & \leq 40S^2\sqrt{8T\beta_T^2\ln\br{3dT^2(T+1)}}  + 1 + 4\max_{c\in[0,1],\nu:\norm{\nu}\leq S}\ell(c,\nu) \tag{$K\geq \lambda T, \delta=1/T$} \\
        & \leq 40S^2\sqrt{8T\beta_T^2\ln\br{3dT^2(T+1)}}  + 1 + 24S^2\br*{R^2d+S^2} + 4\lambda \tag{\Cref{claim:max loss}} \\
        & = \Olog\br*{S^2(R^2d +S^2)\sqrt{T} + \lambda}.
    \end{align*}
\end{proof}

\subsection{Lower Bounds for Known Noise Covariance}
\label{appendix: known covariance lower bound}

We consider the following two one-dimensional instances. In the first instance $p^{-}$, it holds that $x_t\sim \Ncal(0,1-\epsilon)$, while in the second instance $p^{+}$, we have $x_t\sim \Ncal\br*{0,1+\epsilon}$, where $\epsilon\in(0,1/2]$ will be determined later. Aside for this difference, both instances are completely identical; in both, $\theta^*=1$ and $\lambda = 1$, and for any given cost $c$, the feature noise is distributed as $\Ncal(0,\sigma_n^2(c))$ for $\sigma_n^2(c)=\Ind{c<1/2}$. In other words, the noise has unit variance for all costs smaller than $1/2$, and from this payment onward, the noise variance is zero -- there is no noise. 
For this choice of $\theta^*$, we get $y_t=x_t$, and therefore, we can assume w.l.o.g. that at the end of each round, we observe $x_t$ and $n_t$ (so $I_t$ and $I_t^o$ contain the same information). By \Cref{claim: opt values} and \Cref{claim: one-dimension loss}, for feature covariance $\sigma_x^2$ and noise covariance $\sigma_n^2(c)$ (and recalling that $\sigma_{\xhat}^2=\sigma_x^2+\sigma_n^2$), we get an optimal loss
\begin{align*}
    &\ell^*(c) = \sigma_x^2 - \frac{\sigma_x^4}{\sigma_x^2+\sigma_n^2(c)}+\lambda c
    = \frac{\sigma_x^2\sigma_n^2(c)}{\sigma_x^2+\sigma_n^2(c)}+\lambda c.
\end{align*}
Notably, for our choice of noise, the loss is strictly increasing inside the intervals $c\in[0,1/2)$ and $c\in[1/2,1]$ since in these intervals, the noise variance remains the same while the payment increases. Thus, it is always optimal to either play $c=0$ or $c=1/2$. Now, looking at these two specific payments, we have
\begin{align*}
    &\ell_{+}^*(0) = \frac{1+\epsilon}{1+\epsilon +1}=\frac{1+\epsilon}{2+\epsilon}>1/2,\qquad \mathrm{and},
    \qquad\ell_{+}^*(1/2) = 0+\frac{1}{2}=\frac{1}{2}, \\
    &\ell_{-}^*(0) = \frac{1-\epsilon}{1-\epsilon +1}=\frac{1-\epsilon}{2-\epsilon}<1/2, \qquad \mathrm{and},
    \qquad\ell_{-}^*(1/2) = 0+\frac{1}{2}=\frac{1}{2}.
\end{align*}
Therefore, in instance $p^{+}$, an optimal algorithm will play $c^*_+=1/2$, and in instance $p^{-}$, an optimal algorithm will play $c^*_-=0$. The suboptimality of playing the wrong choice is
\begin{align*}
    &\Delta_+ = \ell_{+}^*(0) - \ell_{+}^*(1/2) = \frac{1+\epsilon}{2+\epsilon}-\frac{1}{2} = \frac{\epsilon}{2(2+\epsilon)} \geq \frac{\epsilon}{5},\\
        &\Delta_- = \ell_{-}^*(1/2) - \ell_{-}^*(0) = \frac{1}{2} - \frac{1-\epsilon}{2-\epsilon} = \frac{\epsilon}{2(2-\epsilon)} \geq \frac{\epsilon}{4} \geq \frac{\epsilon}{5},
\end{align*}
where in both cases, the inequalities use the fact that $\epsilon\in(0,1/2]$.

Using these properties allows us to prove a lower bound for the regret:
\lowerBoundKnown*
\begin{proof}
    We adapt the technique of Theorem 15.2 in \citep{lattimore2020bandit}. Let $\Pb_{p^{+}}$ and $\Pb_{p^{-}}$ be the distribution of the entire process of the algorithm when applied to instances $p^{+}$ and $p^{-}$, respectively. Since $\theta^*=1$, we effectively observe $x_t,n_t$ at the end of each round, so the observed history $I_t^o$ and the full one $I_t$ contain the same information. We also remark that $\yhat_t$ is completely determined by the algorithm and does not affect the observation. Indeed, all information gained between $I_t$ and $I_{t+1}$ depends only on $c_{t+1}$ and any internal stochasticity of the algorithm. 
    
    By the chain rule of KL divergence, we have (see. e.g., $(9)$ in \citep{garivier2019explore}
    \begin{align*}
        \kl\br*{\Pb_{p^{+}}^{I_T},\Pb_{p^{-}}^{I_T}}
        = \sum_{t=1}^T \kl\br*{\Pb_{p^{+}}^{\xhat_t,\yhat_t,y_t,x_t,n_t,c_{t+1}\vert I_{t-1}},\Pb_{p^{-}}^{\xhat_t,\yhat_t,y_t,x_t,n_t,c_{t+1}\vert I_{t-1}}}.
    \end{align*}
    Each term can be greatly simplified to 
    \begin{align*}
        \kl&\br*{\Pb_{p^{+}}^{\xhat_t,\yhat_t,y_t,x_t,n_t,c_{t+1}\vert I_{t-1}},\Pb_{p^{-}}^{\xhat_t,\yhat_t,y_t,x_t,n_t,c_{t+1}\vert I_{t-1}}} \\
        &= \kl\br*{\Pb_{p^{+}}^{x_t,n_t\vert I_{t-1}},\Pb_{p^{-}}^{x_t,n_t\vert I_{t-1}}} 
        + \underbrace{\kl\br*{\Pb_{p^{+}}^{\xhat_t,\yhat_t,y_t,c_{t+1}\vert I_{t-1},x_t,n_t},\Pb_{p^{-}}^{\xhat_t,\yhat_t,y_t,c_{t+1}\vert I_{t-1},x_t,n_t}}}_{=0}\\
        &\overset{(1)}{=} \kl\br*{\Pb_{p^{+}}^{x_t,n_t\vert I_{t-1}},\Pb_{p^{-}}^{x_t,n_t\vert I_{t-1}}} 
        \\
        &\overset{(2)}{=} \kl\br*{\Pb_{p^{+}}^{x_t\vert I_{t-1}},\Pb_{p^{-}}^{x_t\vert I_{t-1}}} \\
        &\overset{(3)}{=} \kl\br*{\Ncal(0,1+\epsilon),\Ncal(0,1-\epsilon)}.
    \end{align*}
    In this derivation, relation $(1)$ holds since $\xhat_t,y_t$ are deterministic given $x_t,n_t$ and that $\yhat_t$ and $c_{t+1}$ are determined identically by the algorithm in both instances -- given the same history (and the same $x_t,n_t$), they have the same distribution and the KL divergence equals zero. Relation $(2)$ holds since $n_t$ and $x_t$ are independent given $I_{t-1}$ and $n_t$ has the same distribution in both instances (as $I_{t-1}$ already encapsulates the choice of $c_t$). Finally, $(3)$ substitutes the conditional feature distribution. The resulting KL divergence can be directly bounded by 
    \begin{align*}
        \kl\br*{\Ncal(0,1+\epsilon),\Ncal(0,1-\epsilon)} 
        = \frac{1}{2}\br*{\frac{1+\epsilon}{1-\epsilon} - \ln \frac{1+\epsilon}{1-\epsilon} - 1} 
         \overset{(*)}\leq \frac{1}{4}\br*{\frac{1+\epsilon}{1-\epsilon}-1}^2
         =\frac{1}{4}\br*{\frac{2\epsilon}{1-\epsilon}}^2
        \leq 4\epsilon^2 ,
    \end{align*}
    where $(*)$ use the inequality $\ln(x)\ge x-1-(x-1)^2/2$, which holds for all $x\ge1$, and the last inequality is since $\epsilon \leq 1/2$.     
    Thus, $\kl\br*{\Pb_{p^{+}}^{I_T},\Pb_{p^{-}}^{I_T}} \leq 4T\epsilon^2$. 

    We now use this result to lower-bound the regret. In particular, define the random variable $N_T = \sum_{t=1}^T\Ind{c_t\in[1/2,1]}$ (which is completely determined by the information in $I_T$) and fix $\epsilon = \frac{1}{2\sqrt{T}}\in(0,1/2]$. Then, by the Bretagnolle–Huber inequality (Theorem 14.2 of \citealt{lattimore2020bandit}), it holds that 
    \begin{align}
        \label{eq: lower bound BH}
        \max\brc*{\Pb_{p^{-}}\br*{N_T>T/2},\Pb_{p^{+}}\br*{N_T\leq T/2}} \geq \frac{1}{4}\exp\br*{-\kl\br*{\Pb_{p^{+}}^{I_T},\Pb_{p^{-}}^{I_T}}}
        \geq  \frac{1}{4}\exp\br*{-4T\br*{\frac{1}{2\sqrt{T}}}^2} \geq \frac{1}{12}.
    \end{align}
    Since all distributions are zero-mean Gaussian, these probabilities can be related to the regret in both instances by \Cref{claim: gaussian loss} (and \Cref{claim: one-dimension loss}), namely,
    \begin{align*}
        \Regret_{p^{-}}(T) 
        &\geq \E\brs*{\sum_{t=1}^T \br*{\ell_{-}^*(c_t) - \ell_{-}^*(0)}} \\
        &\geq \E\brs*{\sum_{t=1}^T \br*{\ell_{-}^*(c_t) - \ell_{-}^*(0)}\Ind{c_t\in[1/2,1]}}\\
        &\geq \E\brs*{\sum_{t=1}^T \br*{\ell_{-}^*(1/2) - \ell_{-}^*(0)}\Ind{c_t\in[1/2,1]}}\tag{The loss increases in $[1/2,1]$}\\
        &= \Delta_{-}\E\brs*{\sum_{t=1}^T\Ind{c_t\in[1/2,1]}} \\
        & \geq \Delta_{-}\cdot \frac{T}{2}\Pb_{p^{-}}\br*{N_T>T/2}.
    \end{align*}
    Similarly, we have 
    \begin{align*}
        \Regret_{p^{+}}(T) 
        &\geq \E\brs*{\sum_{t=1}^T \br*{\ell_{+}^*(c_t) - \ell_{+}^*(1/2)}} \\
        &\geq \E\brs*{\sum_{t=1}^T \br*{\ell_{+}^*(c_t) - \ell_{+}^*(1/2)}\Ind{c_t\in[0,1/2)}}\\
        &\geq \E\brs*{\sum_{t=1}^T \br*{\ell_{+}^*(0) - \ell_{+}^*(1/2)}\Ind{c_t\in[0,1/2)}}\tag{The loss increases in $[0,1/2)$}\\
        &= \Delta_{+}\E\brs*{\underbrace{\sum_{t=1}^T\Ind{c_t\in[0,1/2)}}_{=T-N_T}} \\
        & \geq \Delta_{+}\cdot \frac{T}{2}\Pb_{p^{-}}\br*{T-N_T\geq T/2} \\
        & = \Delta_{+}\cdot \frac{T}{2}\Pb_{p^{-}}\br*{N_T\leq T/2}.
    \end{align*}
    Combining both with \Cref{eq: lower bound BH} and lower bounding both gaps by $\epsilon/5$, we get
    \begin{align*}
        \max\brc*{\Regret_{p^{-}}(T),\Regret_{p^{+}}(T)}
        \geq \frac{\epsilon}{5}\cdot\frac{T}{2}\cdot\max\brc*{\Pb_{p^{-}}\br*{N_T>T/2},\Pb_{p^{+}}\br*{N_T\leq T/2}} 
        \geq \frac{\epsilon T}{10}\cdot\frac{1}{12}
        = \frac{\sqrt{T}}{240}.
    \end{align*}
\end{proof}

% \begin{lemma}
%     \label{lemma: log bound}
%     For any $a\ge \frac{1}{2}$ and $x\ge \frac{1}{2a}$, it holds that 
%     \begin{align*}
%         x-\ln(x)-1\le a(1-x)^2.
%     \end{align*}
% \end{lemma}
% \begin{proof}
%     Define the function $f(x)=\ln(1+x)-x+ax^2$. Then, it holds that
%     \begin{align*}
%         f'(x)
%         =\frac{1}{1+x}-1+2ax
%         = 2ax - \frac{x}{1+x}
%         = \frac{x}{1+x}\br*{2ax + 2a-1}.
%     \end{align*}
%     In particular, focusing on the interval $x\ge -\frac{2a-1}{2a} = -1+\frac{1}{2a}$, the function $f$ is decreasing in $[-1+\frac{1}{2a},0]$ and increasing in $x\ge0$, so the minimal value in this interval is at $x=0$. There, $f(0)=0$, and so, $f(x)\ge0$ for all $x\ge -1+\frac{1}{2a}$. Shifting the function by $1$, we get that for all $x\ge \frac{1}{2a}$, it holds that 
%     \begin{align*}
%         \ln(x)\ge x-1-a(x-1)^2.
%     \end{align*}
%     Reorganizing leads to the desired result.
% \end{proof}

\subsection{Regression with Paid Feature and Additive Noise}
\label{appendix: additive noise}
In this appendix, we shortly discuss how the results change if the regression output is noisy, that is, when observing $\tilde{y}_t = x_t^\top\theta^*+\eta_t=y_t+\eta_t$ for some zero-mean noise $\eta_t$ that is $R^2$-subgaussian, conditionally independent of all other quantities at round $t$ and is of variance $\sigma_{\eta}^2$. While we describe the modification only for the case where the noise covariances are known, similar conclusions can be drawn for the other problem variants studied in the paper.

Due to the independence and the fact the the noise is of expectation zero, the instantaneous expected loss can be written as 
\begin{align*}
    \tilde{\ell}(c,\nu) &= \E_{x\sim D_x, n\sim D_n(c),\eta}\brs*{\br*{\br*{x+n}^\top\nu - x^\top\theta^*-\eta}^2} + \lambda c\\
    & = \E_{x\sim D_x, n\sim D_n(c)}\brs*{\br*{\br*{x+n}^\top\nu - x^\top\theta^*}^2}  + \E_{\eta}\brs*{\eta^2} + \lambda c \tag{independence of $\eta$} \\
    & = \ell(c,\nu)+\sigma_{\eta}^2.
\end{align*}
That is, the loss is only shifted by $\sigma_{\eta}^2$, so that the optimal values $\nu^*(c)$ and $c^*$ remain unchanged, and the optimal loss is also biased by the same quantity.

To show that our algorithmic approach still applies, we need to show that the loss estimation, as described in \Cref{prop: loss estimator known noise}, behaves similarly. The new loss estimator that uses the noisy regression outputs can be written as 
\begin{align*}
    \tilde{L}^{kc}_t(c,\nu) 
    &= \sum_{s=1}^t\br*{\br*{\xhat_s^{\top}\nu-\tilde{y}_s}^2 + \nu^{\top}\br*{\Sigma_n(c)-\Sigma_n(c_s)}\nu+\lambda c}\\
    & = \sum_{s=1}^t\br*{\br*{\xhat_s^{\top}\nu-y_s-\eta_s}^2 + \nu^{\top}\br*{\Sigma_n(c)-\Sigma_n(c_s)}\nu+\lambda c}\\
    & =\sum_{s=1}^t\br*{\br*{\xhat_s^{\top}\nu-y_s}^2 + \nu^{\top}\br*{\Sigma_n(c)-\Sigma_n(c_s)}\nu+\lambda c} - 2\sum_{s=1}^t\eta_s\br*{\xhat_s^{\top}\nu-y_s} + \sum_{s=1}^t\eta_s^2\\
    & = \hat{L}^{kc}_t(c,\nu)  - 2\sum_{s=1}^t\eta_sx_s^\top(\nu-\theta^*) -2\sum_{s=1}^t\eta_s n_s^\top\nu + \sum_{s=1}^t\eta_s^2.
\end{align*}
In \Cref{prop: loss estimator known noise}, we already showed that $\hat{L}^{kc}_t(c,\nu)$ converges to the expected loss $t\ell(c,\nu)$. Moreover, the last term does not depend on $c$ or $\nu$, not affecting the minimizers of the loss. Thus, to maintain the same results, one only needs to prove that the two middle terms converge to zero at a comparable rate as the empirical loss to the expected loss. This naturally follows from the fact that the noise is conditionally zero-mean subgaussian given $x_t,n_t$ and the high-probability boundedness of $x_t$ and $n_t$. Alternatively, one can readily extend the analysis technique used in \Cref{prop: loss estimator known noise} and write, for example,
\begin{align*}
    &\sum_{s=1}^t\eta_sx_s^\top(\nu-\theta^*)
    = \begin{pmatrix} 0^\top & 1 \end{pmatrix}\sum_{s=1}^t \br*{\begin{pmatrix} x_s \\ \eta_s \end{pmatrix}\begin{pmatrix} x_s \\ \eta_s \end{pmatrix}^\top}\begin{pmatrix} \nu-\theta^* \\ 0 \end{pmatrix}
    =  \begin{pmatrix} 0^\top & 1 \end{pmatrix}\sum_{s=1}^t \br*{\begin{pmatrix} x_s \\ \eta_s \end{pmatrix}\begin{pmatrix} x_s \\ \eta_s \end{pmatrix}^\top -  \begin{pmatrix} \Sigma_x & 0 \\ 0 & \sigma_{\eta}^2 \end{pmatrix}}\begin{pmatrix} \nu-\theta^* \\ 0 \end{pmatrix}\\
    &\Rightarrow \abs*{\sum_{s=1}^t\eta_sx_s^\top(\nu-\theta^*)} \leq 2S\norm*{\sum_{s=1}^t \br*{\begin{pmatrix} x_s \\ \eta_s \end{pmatrix}\begin{pmatrix} x_s \\ \eta_s \end{pmatrix}^\top -  \begin{pmatrix} \Sigma_x & 0 \\ 0 & \sigma_{\eta}^2 \end{pmatrix}}}_{op},
\end{align*}
where the operator norm can be bounded by \Cref{lemma: matrix azuma for subgaussians}, leading to similar rates.

To summarize, one can easily extend the concentration bounds in \Cref{prop: loss estimator known noise} to the case of noisy outputs, with only mild modifications to the constants. Since this is the concentration property that we use for both regret proofs (either directly or via \Cref{corollary: loss estimator unknown noise}), the same constant change will propagate to the optimistic index and regret bounds, leading to very minor algorithmic and bound changes.
\clearpage

\section{Proofs for Unknown Noise Covariances}
\label{appendix: unknown covariances}
\lossEstUnknownNoise*
\begin{proof}
    Before the interaction starts, we sample $T$ i.i.d. samples $\brc*{z_t}_{t\in[T]}$ and $\brc*{\eta_t}_{t\in[T]}$ from the distributions $D_x$ and $D_n(c)$. Then, without loss of generality, we assume that upon choosing $c_t=c$ for the $i^{th}$ time, the feature vector is taken as $x_t=z_i$ and the noise vector is given by $n_t=\eta_i$, so that the regression output is $y_t=x_t^\top\theta^*=z_i^\top\theta^*\triangleq \xi_i$. Then, defining the loss
    \begin{align*}
        \tilde{L}_i(c,\nu) = \sum_{s=1}^i\br*{\br*{\br*{z_s+\eta_s}^{\top}\nu-\xi_i}^2 + \lambda c},
    \end{align*}
    we have that $\hat{L}^{uc}_t(c,\nu) = \tilde{L}_{N_t(c)}(c,\nu)$.
        
    Now, by \Cref{prop: loss estimator known noise}, we have that for all $i\ge1$ and $\nu\in\R^d$ s.t. $\norm*{\nu}\le S$, it holds w.p. $1-3\delta$ that 
    \begin{align*}
        \abs*{\tilde{L}_i(c,\nu) - i\ell(c,\nu)} \leq  9S^2\sqrt{8i\beta_i^2\ln\frac{3di(i+1)}{\delta}}.
    \end{align*}
    Since the event holds simultaneously for all $i\ge1$, it also holds for $i=N_t(c)$ for all $t\ge1$. In particular, w.p. $1-3\delta$, for all $t\ge1$  and $\nu\in\R^d$ s.t. $\norm*{\nu}\le S$, 
    \begin{align*}
        \abs*{\hat{L}^{uc}_t(c,\nu) - N_t(c)\ell(c,\nu)} \leq 9S^2\sqrt{8N_t(c)\beta_{N_t(c)}^2\ln\frac{3dN_t(c)(N_t(c)+1)}{\delta}}
        \leq 9S^2\sqrt{8N_t(c)\beta_{t}^2\ln\frac{3dt(t+1)}{\delta}}.
    \end{align*}
    where we used the fact that $N_t(c)\leq t$ and the monotonicity of $\beta_t$ in $t$. 
\end{proof}

\regretUnknownNoiseCov*
\begin{proof}
    Define the good event $\Gcal_k$ under which all the results of \Cref{corollary: loss estimator unknown noise} hold for $c=k/K$ and let $\Gcal=\bigcap_{k\in[K]}\Gcal_k$. By the corollary, for $\delta = \frac{1}{KT},$ $\Pr\brc*{\Gcal}\geq 1-3/T$, and under this event, for all $t\ge1$, $c\in\brc*{k/K}_{k\in[K]}$ and $\nu$ s.t. $\norm{\nu}\leq S$, 
    \begin{align}
        \label{eq: good event unknown covariances}
        \abs*{\frac{1}{N_t(c)}\hat{L}^{uc}_t(c,\nu) - \ell(c,\nu)} \leq  9S^2\sqrt{\frac{8\beta_t^2\ln\frac{3dt(t+1)}{\delta}}{N_t(c)}}.
    \end{align}
    As a first step, we relate the empirical loss at the selected cost and linear regressor to the real loss when $\G$ holds.
    \begin{align*}
        \hat{L}^{uc}_t\br*{\frac{k_{t}}{K},\hat{\nu}_{t}(k_{t})}
        & = \br*{\hat{L}^{uc}_t\br*{\frac{k_{t}}{K},\hat{\nu}_{t}(k_{t})}-9S^2\sqrt{8N_t\br*{\frac{k_t}{K}}\beta_t^2\ln\frac{3dt(t+1)}{\delta}}} 
        + 9S^2\sqrt{8N_t\br*{\frac{k_t}{K}}\beta_t^2\ln\frac{3dt(t+1)}{\delta}} \\
        & = N_t\br*{\frac{k_t}{K}}\min_{k\in[K],\nu:\norm*{\nu}\leq S}\brc*{\frac{\hat{L}^{uc}_t\br*{k/K,\nu}}{N_t(k/K)}-9S^2\sqrt{\frac{8\beta_t^2\ln\frac{3dt(t+1)}{\delta}}{N_t(k/K)}}} 
        +9 S^2\sqrt{8N_t\br*{\frac{k_t}{K}}\beta_t^2\ln\frac{3dt(t+1)}{\delta}}.
    \end{align*}
    In particular, denoting $k^*\in\argmin_{k\in[K]} \ell^*(k/K)$, we can bound
    \begin{align}
        \hat{L}^{uc}_t\br*{\frac{k_{t}}{K},\hat{\nu}_{t}(k_{t})}
        & \leq N_t\br*{\frac{k_t}{K}}\br*{\frac{\hat{L}^{uc}_t\br*{k^*/K,\nu^*(k^*/K)}}{N_t(k^*/K)} - 9S^2\sqrt{\frac{8\beta_t^2\ln\frac{3dt(t+1)}{\delta}}{N_t(k^*/K)}}}
        + 9S^2\sqrt{8N_t\br*{\frac{k_t}{K}}\beta_t^2\ln\frac{3dt(t+1)}{\delta}} \nonumber\\
        & \leq N_t\br*{\frac{k_t}{K}}\ell\br*{\frac{k^*}{K},\nu^*\br*{\frac{k^*}{K}}} + 9S^2\sqrt{8N_t\br*{\frac{k_t}{K}}\beta_t^2\ln\frac{3dt(t+1)}{\delta}} \tag{By \cref{eq: good event unknown covariances}} \\
        & = N_t\br*{\frac{k_t}{K}}\ell^*\br*{\frac{k^*}{K}} + 9S^2\sqrt{8N_t\br*{\frac{k_t}{K}}\beta_t^2\ln\frac{3dt(t+1)}{\delta}} \label{eq: empirical to optimal loss unknown covariances}.
    \end{align}
    Combined with \cref{eq: good event unknown covariances}, we can leverage this inequality to relate the expected loss at the played cost/prediction to the optimal expected loss, namely 
    \begin{align*}
        \ell\br*{\frac{k_{t}}{K},\hat{\nu}_{t}(k_{t})}
        &\leq \frac{\hat{L}^{uc}_t\br*{k_{t}/K,\hat{\nu}_{t}(k_{t})}}{N_t(k/K)} + 9S^2\sqrt{\frac{8\beta_t^2\ln\frac{3dt(t+1)}{\delta}}{N_t(k_t/K)}} \tag{By \cref{eq: good event unknown covariances}}\\
        & \leq \ell^*\br*{\frac{k^*}{K}} + 18S^2\sqrt{\frac{8\beta_t^2\ln\frac{3dt(t+1)}{\delta}}{N_t(k_t/K)}} \tag{By \cref{eq: empirical to optimal loss unknown covariances}}\\
        & \leq \ell^* + \frac{\lambda}{K} + 18S^2\sqrt{\frac{8\beta_t^2\ln\frac{3dt(t+1)}{\delta}}{N_t(k_t/K)}}, 
    \end{align*}
    where the last transition is by \cref{eq: global to discretized minimizer}. 
    
    We are finally ready to bound the regret, using the regret decomposition of \Cref{lemma: regret decomposition}:
    \begin{align*}
        \Regret(T) 
        & = \E\brs*{\sum_{t=1}^T \br*{\ell(c_t,\nu_t) - \ell^*}} \\
        & = \E\brs*{\sum_{t=1}^T \br*{\ell\br*{\frac{k_{t-1}}{K},\hat{\nu}_{t-1}(k_{t-1})} - \ell^*}} \\
        & \leq \E\brs*{\sum_{t=K+1}^T \br*{\ell\br*{\frac{k_{t-1}}{K},\hat{\nu}_{t-1}(k_{t-1})} - \ell^*}} + K\max_{c\in[0,1],\nu:\norm{\nu}\leq S}\ell(c,\nu) \tag{separating $T\leq K$} \\
        & \leq  \E\brs*{\Ind{\Gcal}\sum_{t=K+1}^T \br*{\ell\br*{\frac{k_{t-1}}{K},\hat{\nu}_{t-1}(k_{t-1})} - \ell^*}} + \underbrace{\Pr\brc*{\bar{\Gcal}}T}_{\leq 3}\max_{c\in[0,1],\nu:\norm{\nu}\leq S}\ell(c,\nu) + K\max_{c\in[0,1],\nu:\norm{\nu}\leq S}\ell(c,\nu) \\
        & \leq \sum_{t=K+1}^{T}\br*{18S^2\sqrt{\frac{8\beta_t^2\ln\frac{3dt(t+1)}{\delta}}{N_{t-1}(k_{t-1}/K)}}  +\frac{\lambda}{K}} + 4K\max_{c\in[0,1],\nu:\norm{\nu}\leq S}\ell(c,\nu).
    \end{align*}
    The sum over the first term can be bounded using standard bandit arguments -- noticing that the counts start from $1$ due to the initial sampling stage and increase by one every time we play $k_t=k$, we have
    \begin{align*}
        \sum_{t=K+1}^T \frac{1}{\sqrt{N_{t-1}(k_{t-1}/K)}}
        &= \sum_{k\in[K]}\sum_{t=K+1}^T \frac{\Ind{k_{t-1}=k}}{\sqrt{N_{t-1}(k/K)}}
        \leq \sum_{k\in[K]}\sum_{i=1}^{N_T(k/K)}\frac{1}{\sqrt{i}}
        \leq 2\sum_{k\in[K]}\sqrt{N_T(k/K)} \\
        & \overset{(*)}\leq 2\sqrt{K}\sqrt{\sum_{k\in[K]}N_T(k/K)}
        = 2\sqrt{KT},
    \end{align*}
    where inequality $(*)$ is by Cauchy-Schwarz. Substituting back while noting that $\beta_t$ increases with $t$ yields
    \begin{align*}
        \Regret(T) 
        & \leq 36S^2\sqrt{8KT\beta_T^2\ln\frac{3dT(T+1)}{\delta}} + \frac{\lambda T}{K} + 4K\max_{c\in[0,1],\nu:\norm{\nu}\leq S}\ell(c,\nu) \\
        & \leq 36S^2\sqrt{8KT\beta_T^2\ln\br*{3dKT^2(T+1)}} + \frac{\lambda T}{K} + 4K\br*{6S^2\br*{R^2d+S^2} + \lambda} \tag{\Cref{claim:max loss}}\\
        & = \Olog\br*{\br*{ S^2(R^2d+S^2)}^{2/3}\lambda^{1/3} T^{2/3}},
    \end{align*}
where the last inequality is due to the specific choice of $K$ and the fact that $\beta_T = \Olog\br*{R^2d+S^2}$.
\end{proof}
\clearpage

\subsection{Lower Bounds for Unknown Noise Covariance}
\label{appendix: unknown covariances lower bound}
We consider the following one-dimensional example. Let $x_t\sim \Ncal(0,1)$, $\theta^*=1$ and $\lambda = \frac{1}{2}$. For any given cost $c$, the feature noise is distributed as $\Ncal(0,\sigma_n^2(c))$; we will briefly specify $\sigma_n^2(c)$. Notably, $y_t=x_t$, and therefore, we can assume w.l.o.g. that at the end of each round, we observe $x_t$ and $n_t$. By \Cref{claim: opt values} and \Cref{claim: one-dimension loss}, we have
\begin{align*}
    \ell^*(c) = 1 - \frac{1}{1+\sigma_n^2(c)}+\frac{1}{2}c
    = \frac{\sigma_n^2(c)}{1+\sigma_n^2(c)}+\frac{1}{2}c.
\end{align*}
Finally, we denote $f(c) = \frac{1-c}{1+c}$. Notice that if $\sigma_n^2(c)=f(c)$, it holds that 
\begin{align}
    \label{eq: optimal loss default profile lower bound unknown}
    \ell^*(c) = \frac{\frac{1-c}{1+c}}{1+\frac{1-c}{1+c}}+\frac{1}{2}c
    = \frac{1-c}{2}+\frac{1}{2}c = \frac{1}{2}.
\end{align}

Let $K\in\N$ and consider the following instances:
\begin{enumerate}
    \item The baseline instance $p$ is such that $\sigma_n^2(c) = f(c)$, and so $\ell^*(c)=\frac{1}{2}$.
    \item For any $k\in[K]$, define $c_k = \frac{1}{2}+\frac{k-1}{4K}$ (with $c_{K+1}=\frac{3}{4}$)
    . We then define the instance $p_k$ with the noise variance
    \begin{align*}
        \sigma_n^2(c;k) = \begin{cases}
			f(c_k) + 2\frac{c-c_k}{c_{k+1}-c_k}\br*{f(c_{k+1})-f(c_k)}
			&c\in\left[c_k, \frac{c_k+c_{k+1}}{2} \right)\\
			f(c_{k+1})
			&c\in\left[\frac{c_k+c_{k+1}}{2}, c_{k+1}  \right)\\
			f(c)
			& \text{ else }\\
		\end{cases}
    \end{align*}
    Since $f(c)$ is $1$-Lipschitz and decreasing, with $f(c)\in\brs*{\frac{1}{8},\frac{1}{2}}$ for all $c\in\brs*{\frac{1}{2},\frac{3}{4}}$, then $\sigma_n^2(c;k)$ is both decreasing and $1$-Lipschitz.
\end{enumerate}
This choice enjoys the following properties  (which we formally prove after proving the lower bound theorem):
\begin{restatable}{claim-rst}{lowerBoundUnknownProp}
    \label{claim: lower bound properties}
    \begin{enumerate}
        \item For any $c,k$, the KL divergence between the two noise distributions is bounded by 
        \begin{align*}
            \kl\br*{\Ncal(0,\sigma_n^2(c)),\Ncal(0,\sigma_n^2(c;k))} \leq 16\br*{c_{k+1} - c_k}^2\Ind{c\in\left[c_k,c_{k+1}\right)]}
            = \frac{1}{K^2}\Ind{c\in\left[c_k,c_{k+1}\right)]}.
        \end{align*}
        \item The optimal loss in instance $p_k$ is bounded by
        \begin{align*}
            \min_{c\in\brs*{0,1}}\ell^*(c;k) \leq \frac{1}{2} - \frac{c_{k+1}-c_k}{4}
            = \frac{1}{2} - \frac{1}{16K}.
        \end{align*}
    \end{enumerate}
\end{restatable}
Using these properties allows us to prove a lower bound for the regret:
\lowerBoundUnknown*
\begin{proof}
    We adapt the technique of Theorem 6 in \citep{garivier2019explore}. Fix some $K\in\N$ that will be determined later.  
    Let $\Pb_p$ and $\Pb_{p_k}$ be the distribution of the entire process of the algorithm when applied to instances $p$ and $p_k$, respectively. As in \Cref{theorem: known covariances lower bound}, notice that the choice $\theta^*=1$ renders $x_t,n_t$ completely identifiable at the end of each round from $y_t,\xhat_t$, so that $I_t^o$ and $I_t$ contain the same information. We also remark that $\yhat_t$ is completely determined by the algorithm and has no effect on the observation; therefore, all the new information between $I_t$ and $I_{t+1}$ only depends on $c_{t+1}$.

    By the chain rule of KL divergence, for any $k\in[K]$, we have (see. e.g., eq. $(9)$ in \citealt{garivier2019explore})
    \begin{align*}
        \kl\br*{\Pb_p^{I_T},\Pb_{p_k}^{I_T}}
        = \sum_{t=1}^T \kl\br*{\Pb_p^{\xhat_t,\yhat_t,y_t,x_t,n_t,c_{t+1}\vert I_{t-1}},\Pb_{p_k}^{\xhat_t,\yhat_t,y_t,x_t,n_t,c_{t+1}\vert I_{t-1}}}.
    \end{align*}
    Each term can be greatly simplified to 
    \begin{align*}
        \kl&\br*{\Pb_p^{\xhat_t,\yhat_t,y_t,x_t,n_t,c_{t+1}\vert I_{t-1}},\Pb_{p_k}^{\xhat_t,\yhat_t,y_t,x_t,n_t,c_{t+1}\vert I_{t-1}}} \\
        &= \kl\br*{\Pb_p^{x_t,n_t\vert I_{t-1}},\Pb_{p_k}^{x_t,n_t\vert I_{t-1}}} 
        + \underbrace{\kl\br*{\Pb_p^{\xhat_t,\yhat_t,y_t,c_{t+1}\vert I_{t-1}x_t,n_t},\Pb_{p_k}^{\xhat_t,\yhat_t,y_t,c_{t+1}\vert I_{t-1}x_t,n_t}}}_{=0}\\
        &\overset{(1)}{=} \kl\br*{\Pb_p^{x_t,n_t\vert I_{t-1}},\Pb_{p_k}^{x_t,n_t\vert I_{t-1}}} 
        \\
        &\overset{(2)}{=} \kl\br*{\Pb_p^{n_t\vert I_{t-1}},\Pb_{p_k}^{n_t\vert I_{t-1}}} \\
        &\overset{(3)}{=} \E_p\brs*{\kl\br*{\Ncal(0,\sigma_n^2(c_t)),\Ncal(0,\sigma_n^2(c_t;k))}} \\
        &\overset{(4)}{\leq}  \frac{1}{K^2}\E_p\brs*{\Ind{c_t\in\left[c_k,c_{k+1}\right)]}}.
    \end{align*}
    In this derivation, relation $(1)$ holds since given the same history and $x_t,y_t$, the variables $\xhat_t,y_t$ are deterministic, so their KL is zero. Moreover, given the same quantities, $\yhat$ and $c_{t+1}$ are identically determined by the algorithm in both instances -- have the same conditional distribution -- and the KL divergence is equal to zero. Relation $(2)$ holds since $n_t$ and $x_t$ are independent given $I_{t-1}$ and $x_t$ has the same distribution in both instances. $(3)$ substitutes the conditional noise distribution and $(4)$ is by \Cref{claim: lower bound properties}.

    Denoting $N_T(k) = \sum_{t=1}^T \Ind{c_t\in\left[c_k,c_{k+1}\right)]}$, we can thus bound
    \begin{align*}
        \kl\br*{\Pb_p^{I_T},\Pb_{p_k}^{I_T}}
        \leq \frac{1}{K^2}\E_p\brs*{N_T(k)}.
    \end{align*}
    On the other hand, by the data processing inequality (e.g., eq. $(8)$ in \citealt{garivier2019explore}) it holds that 
    \begin{align*}
        \kl\br*{\Pb_p^{I_T},\Pb_{p_k}^{I_T}}
        \geq \klBin\br*{\frac{\E_p\brs*{N_T(k)}}{T},\frac{\E_{p_k}\brs*{N_T(k)}}{T}},
    \end{align*}
    where $\klBin(p,q)=p\ln\frac{p}{q} +(1-p)\ln\frac{1-p}{1-q}$. Combining both inequalities and using Pinsker's inequality, we get 
    \begin{align*}
        \frac{1}{K^2}\E_p\brs*{N_T(k)}
        \geq 2\br*{\frac{\E_p\brs*{N_T(k)}}{T}-\frac{\E_{p_k}\brs*{N_T(k)}}{T}}^2,
    \end{align*}
    and reorganizing leads to the bound
    \begin{align*}
        \E_{p_k}\brs*{N_T(k)} \leq \frac{T}{K}\sqrt{\frac{\E_p\brs*{N_T(k)}}{2}} + \E_p\brs*{N_T(k)}.
    \end{align*}

    We are now ready to choose an instance $k$. Since the intervals $[c_k,c_{k+1})$ are disjoint and $\sum_{k\in[K]}N_T(k)\leq T$, by the pigeonhole principle, there exists $k^*$ such that $\E_p\brs*{N_T(k^*)} \leq \frac{T}{K}$. For this instance, choosing $K=\ceil{2T^{1/3}}$ it holds that
    \begin{align*}
        \E_{p_{k^*}}\brs*{N_T(k^*)} \leq \frac{1}{\sqrt{2}}\br*{\frac{T}{K}}^{3/2} + \frac{T}{K} 
        \leq \frac{3T}{4}.
    \end{align*}
    Moreover, by \Cref{claim: lower bound properties}, we know that 
    $$\min_{c\in[0,1]} \ell^*(c;k) \le \frac{1}{2}-\frac{1}{16K} \le \frac{1}{2} - \frac{T^{-1/3}}{64},$$
    while outside the interval $[c_k,c_{k+1})$, the loss is $\ell^*(c;k)=\frac{1}{2}$. Therefore, using our zero-mean Gaussian choice for $x_t,n_t$, by \Cref{claim: gaussian loss}  (and \Cref{claim: one-dimension loss}), the regret under instance $k^*$ is lower bounded by 
    \begin{align*}
        \E_{p_{k^*}}\brs*{R(T)}
        &\geq \E_{p_{k^*}}\brs*{\sum_{t=1}^T \ell^*(c_t;k^*) - \min_{c\in[0,1]} \ell^*(c;k^*)} \\
        & \geq \E_{p_{k^*}}\brs*{\sum_{t=1}^T \Ind{c_t\notin\left[c_{k^*},c_{k^*+1}\right)]}\br*{\ell^*(c_t;k^*) - \min_{c\in[0,1]} \ell^*(c;k^*)}} \\
        & \geq \E_{p_{k^*}}\brs*{\sum_{t=1}^T \Ind{c_t\notin\left[c_{k^*},c_{k^*+1}\right)]}\frac{T^{-1/3}}{64}} \\
        & = \frac{T^{-1/3}}{64}\br*{T - \E_{p_{k^*}}\brs*{N_T(k^*)}} \\
        & \geq \frac{T^{2/3}}{256}.
    \end{align*}
\end{proof}
\clearpage

\lowerBoundUnknownProp*
\begin{proof}
We start by stating a few properties of the noise variances under the two instances.
\begin{itemize}
    \item Outside the interval $[c_k,c_{k+1}]$ it holds that $\sigma_n^2(c)=\sigma_n^2(c;k)$; in particular, $$\kl\br*{\Ncal(0,\sigma_n^2(c)),\Ncal(0,\sigma_n^2(c;k))} =0.$$
    \item We can rewrite the variance $\sigma_n^2(c)=f(c)=\frac{2}{1+c}-1$. Specifically, the variance slope $\frac{d\sigma_n^2(c)}{dc}=-\frac{2}{(1+c)^2}$ is increasing (becoming less negative) as $c$ increases. Moreover, for $c\in\left[c_k, \frac{c_k+c_{k+1}}{2} \right)$, it holds that
    \begin{align*}
        \frac{d\sigma_n^2(c;k)}{dc}=\frac{2}{c_{k+1}-c_k}\br*{f(c_{k+1})-f(c_k)}
        &= \frac{2}{c_{k+1}-c_k}\br*{\frac{2}{1+c_{k+1}}-\frac{2}{1+c_k}}
        = -\frac{4}{\br*{1+c_{k+1}}\br*{1+c_{k}}} \\
        & \leq -\frac{2}{\br*{1+c_{k}}^2},
    \end{align*}
    where the last inequality is since $c_k,c_{k+1}\in[0,1]$. In other words, we have that $\sigma_n^2(c_k) = \sigma_n^2(c_k;k)$ and $\frac{d\sigma_n^2(c;k)}{dc} \leq \frac{d\sigma_n^2(c_k)}{dc}\leq \frac{d\sigma_n^2(c)}{dc}$ for all $c\in\left[c_k, \frac{c_k+c_{k+1}}{2} \right)$, and so inside this interval, the difference $\sigma_n^2(c)-\sigma_n^2(c;k)$ is non-negative and maximized at $c=\frac{c_k+c_{k+1}}{2}$. This cost is also the maximizer of the difference in $c\in\left[c_k, c_{k+1} \right)$, since inside the interval $c\in\left[\frac{c_k+c_{k+1}}{2},c_{k+1} \right)$, $\sigma_n^2(c)$ decreases while $\sigma_n^2(c;k)$ remains constant. We can similarly conclude that $\sigma_n^2(c)-\sigma_n^2(c;k)\ge0$ for $c\in\left[\frac{c_k+c_{k+1}}{2},c_{k+1} \right)$, since $\sigma_n^2(c)\ge \sigma_n^2(c_{k+1})=\sigma_n^2(c;k)$; therefore, the difference is non-negative across the interval $\left[c_k, c_{k+1} \right)$. To summarize, we showed that for any $c\in \left[c_k, c_{k+1} \right)$, it holds that 
    \begin{align*}
        0 \leq \sigma_n^2(c)-\sigma_n^2(c;k) \leq \sigma_n^2\br*{\frac{c_k+c_{k+1}}{2}}-\sigma_n^2\br*{\frac{c_k+c_{k+1}}{2};k} = \sigma_n^2\br*{\frac{c_k+c_{k+1}}{2}}-\sigma_n^2(c_{k+1}).
    \end{align*}
    We now explicitly bound the difference
    \begin{align*}
        \sigma_n^2\br*{\frac{c_k+c_{k+1}}{2}}-\sigma_n^2(c_{k+1})
        & = \frac{2}{1+\frac{c_k+c_{k+1}}{2}} - \frac{2}{1+c_{k+1}} 
        = \frac{c_{k+1} - c_k}{\br*{1+\frac{c_k+c_{k+1}}{2}}\br*{1+c_{k+1}}} 
        \leq c_{k+1} - c_k
    \end{align*}
    To conclude this property, we have for all $c\in\left[c_k, c_{k+1} \right)$ that
    \begin{align}
        \label{eq:var difference}
        0 \leq 
        \sigma_n^2(c)-\sigma_n^2(c;k)
        \leq c_{k+1} - c_k.
    \end{align}
    \item As stated near its definition, since $f(c)$ is decreasing in $c$, it can be easily verified that $\sigma_n^2(c;k)$ is also decreasing, and so for any $c\in\left[c_k,c_{k+1}\right)$, it holds that 
    $$\sigma_n^2(c;k) \geq \sigma_n^2(c_{k+1};k) = f(c_{k+1}) \geq f\br*{\frac{3}{4}} \geq \frac{1}{8}.$$
\end{itemize}

Using these calculations, we are now ready to bound the KL divergence between the two distributions:
\begin{align*}
    \kl\br*{\Ncal(0,\sigma_n^2(c)),\Ncal(0,\sigma_n^2(c;k))} 
    &= \frac{1}{2}\br*{\frac{\sigma_n^2(c)}{\sigma_n^2(c;k)} - \ln \frac{\sigma_n^2(c)}{\sigma_n^2(c;k)} - 1}\Ind{c\in\left[c_k,c_{k+1}\right)]} \\
    & \leq \frac{\br*{\frac{\sigma_n^2(c)}{\sigma_n^2(c;k)}-1}^2}{4}\Ind{c\in\left[c_k,c_{k+1}\right)]} \tag{$\ln x \geq x - 1 -\frac{(x-1)^2}{2}$ for $x\ge 1$} \\
    & = \frac{\br*{\sigma_n^2(c) - \sigma_n^2(c;k)}^2}{4\sigma_n^4(c;k)}\Ind{c\in\left[c_k,c_{k+1}\right)]} \\
    & \leq 16\br*{c_{k+1} - c_k}^2\Ind{c\in\left[c_k,c_{k+1}\right)]},
\end{align*}
where the last inequality is by \cref{eq:var difference} and the lower bound $\sigma_n^2(c;k) \geq\frac{1}{8}$. This concludes the first part of the claim. 

For the second part, we have 
\begin{align*}
    \min_{c\in[0,1]}\ell^*(c;k) 
    &\leq \ell^*\br*{\frac{c_k+c_{k+1}}{2};k}\\
    &= \frac{\sigma_n^2\br*{\frac{c_k+c_{k+1}}{2};k}}{1+\sigma_n^2\br*{\frac{c_k+c_{k+1}}{2};k}}+\frac{1}{2}\frac{c_k+c_{k+1}}{2}\\
    & = \frac{f\br*{c_{k+1}}}{1+f\br*{c_{k+1}}}+\frac{c_k+c_{k+1}}{4} \\
    & = \frac{1}{2} - \frac{c_{k+1}}{2} + \frac{c_k+c_{k+1}}{4}\tag{by \cref{eq: optimal loss default profile lower bound unknown}, $\frac{f\br*{c}}{1+f\br*{c}}+\frac{c}{2} = \frac{1}{2}$} \\
    & = \frac{1}{2} - \frac{c_{k+1}-c_k}{4}.
\end{align*}
\end{proof}

\clearpage

\section{Useful Concentration Results}
\label{appendix: concentration}

We start by stating two existing concentration results and then prove a lemma that combines both of them.
\begin{lemma}[Matrix Azuma, \citealt{tropp2012user}, Theorem 7.1]
    \label{lemma: matrix azuma}
    Let $M_1,M_2,\dots,M_t\in\R^{d\times d}$ a sequence of self-adjoint matrices adapted to a filtration $\F_t$ such that $\E\brs*{M_t\vert \F_{t-1}}=0$, and assume that there exist a fixed sequence of self-adjoint matrices $A_1,\dots A_t$ such that $M_s^2\preceq A_s^2$ a.s. for all $s$. Finally, denote $\sigma^2=\norm*{\sum_{s=1}^t A_s^2}_{op}$. Then, for all $t\ge1$ and $\delta>0$,
    \begin{align*}
        \Pr\brc*{\lambda_{\max}\br*{\sum_{s=1}^t M_s}\ge \sqrt{8\sigma^2\ln\frac{d}{\delta}}} \leq \delta.
    \end{align*}
\end{lemma}
\begin{lemma}[Concentration of Subgaussian Norm, \citealt{hsu2012tail}, Theorem 2.1 with $A=I$ and Remark 2.2]
     \label{lemma: subgaussian norm}
    Let $X\in\R^d$ be a random $R^2$-subgaussian vector of mean $\E[X]=\mu$ such that 
    \begin{align*}
        \forall \alpha\in\R^d,\qquad \E\brs*{e^{\alpha^\top(X-\mu)}} \leq e^{\norm{\alpha}^2R^2/2}.
    \end{align*}
    Then, for all $\delta>0$,
    \begin{align*}
        \Pr\brc*{\norm*{X}^2 \geq R^2 \br*{d +2\sqrt{d\ln\frac{1}{\delta}}+2\ln\frac{1}{\delta}} + \norm*{\mu}^2 \br*{1 + 2\sqrt{\frac{\ln\frac{1}{\delta}}{d}}}} \leq \delta.
    \end{align*}
    and $\E\brs*{\norm*{X_t-\mu}^2}\leq R^2d$.
\end{lemma}

\begin{lemma}
     \label{lemma: matrix azuma for subgaussians}
     Let $X_1,\dots,X_t\in\R^d$ be a sequence of $R^2$-subgaussians vectors adapted to the filtration $\F_t$ such that $\E[X_t\vert \F_{t-1}]=\mu_t$, $\E\brs*{X_tX_t^\top\vert \F_{t-1}}=\Sigma_t$ and 
     \begin{align*}
        \forall \alpha\in\R^d,\qquad \E\brs*{e^{\alpha^\top(X_t-\mu_t)}\vert \F_{t-1}} \leq e^{\norm{\alpha}^2R^2/2}.
    \end{align*}
    Further assume that $\norm*{\mu_t}\leq S$ a.s. for all $t\ge1$. Then, with probability at least $1-\delta$, for all $t\ge1$, it holds that
    \begin{align*}
        \norm*{\sum_{s=1}^t X_sX_s^\top-\Sigma_t}_{op} < \sqrt{8t\beta_t^2\ln\frac{3dt(t+1)}{\delta}} .
    \end{align*}
    where $\beta_t = R^2 \br*{d +2\sqrt{d\ln\frac{3t(t+1)}{\delta}}+2\ln\frac{3t(t+1)}{\delta}} + S^2 \br*{1 + 2\sqrt{\frac{\ln\frac{3t(t+1)}{\delta}}{d}}}$.
\end{lemma}
\begin{proof}
    Throughout the proof, we only work with symmetric matrices, for which $\norm{A}_{op}= \max\br*{\lambda_{\max}\br*{A},-\lambda_{\min}\br*{A}}$.

    We consider the matrix $M_t=X_tX_t^\top-\Sigma_t$ and aim to bound its maximal and minimal eigenvalues with high probability.
    \begin{itemize}
        \item To bound the minimal eigenvalue, notice that $M_t \succeq -\Sigma_t$ a.s., so we focus on bounding the maximal singular value of $\Sigma_t$. Specifically, we have
        \begin{align*}
            \lambda_{\max}\br*{\Sigma_t}
            &= \lambda_{\max}\br*{\E\brs*{X_tX_t^\top\vert \F_{t-1}}}
            = \max_{u\in\R^d: \norm{u}\leq1} u^\top \E\brs*{X_tX_t^\top\vert \F_{t-1}} u
            \leq \E\brs*{\max_{u\in\R^d: \norm{u}\leq1} u^\top X_tX_t^\top u\vert \F_{t-1}}\\
            &= \E\brs*{\norm*{X_t}^2\vert \F_{t-1}}.
        \end{align*}
        In addition, by \Cref{lemma: subgaussian norm}, we have 
        \begin{align*}
            R^2d \geq \E\brs*{\norm*{X_t-\mu_t}^2\vert \F_{t-1}} = \E\brs*{\norm*{X_t}^2\vert \F_{t-1}} - 2 \E[X_t\vert \F_{t-1}]^\top\mu_t +\norm*{\mu_t}^2 = \E\brs*{\norm*{X_t}^2\vert \F_{t-1}} - \norm*{\mu_t}^2,
        \end{align*}
        and therefore, 
        $$\lambda_{\max}\br*{\Sigma} \leq R^2d + \norm*{\mu_t}^2\leq R^2d + S^2,$$
        or $\lambda_{\min}\br*{M_t} \geq -\br*{R^2d + S^2}\geq -\beta_t$ a.s. for all $t\geq 1$.
        \item To bound the maximal eigenvalue, define the event
        \begin{align*}
            E_n = \brc*{\forall t\ge1: \norm*{X_t}^2 \leq \beta_t}.
        \end{align*}
        By \Cref{lemma: subgaussian norm} and the union bound, it holds that
        \begin{align*}
            \Pr\brc*{E_n} &= 1 - \Pr\brc*{\exists t\ge1: \norm*{X_t}^2 > \beta_t}\\
            &\geq 1 - \sum_{t=1}^{\infty}\Pr\brc*{\norm*{X_t}^2 > \beta_t}\\
            &\geq 1 - \sum_{t=1}^{\infty}\Pr\brc*{\norm*{X_t}^2 > R^2 \br*{d +2\sqrt{d\ln\frac{3t(t+1)}{\delta}}+2\ln\frac{3t(t+1)}{\delta}} + \norm*{\mu_t}^2 \br*{1 + 2\sqrt{\frac{\ln\frac{3t(t+1)}{\delta}}{d}}}} \tag{$\norm{\mu_t}\le S$}\\
            &\geq 1-\sum_{t=1}^{\infty}\frac{\delta}{3t(t+1)}\\
            &= 1-\frac{\delta}{3}.
        \end{align*}
        Under $E_n$, since $\Sigma_t$ is positive semi-definite, we have that 
        \begin{align*}
            \lambda_{\max}\br*{M_t}
            \leq\lambda_{\max}\br*{X_tX_t^\top}
            = \max_{u\in\R^d: \norm{u}\leq1} u^\top X_tX_t^\top u
            = \norm*{X_t}^2
            \leq \beta_t.
        \end{align*}
    \end{itemize}  
    To summarize, under $E_n$, the operator norm of $M_t$ is $\norm*{M_t}_{op}\leq \beta$; therefore, we have that a.s., $\norm*{M_t^2\Ind{E_n}}_{op}\leq \beta^2$ (or $\br*{M_t\Ind{E_n}}^2\preceq \beta_t^2 I$). We now apply \Cref{lemma: matrix azuma} on the sequences $M_t^+=M_t\Ind{E_n}$ and $M_t^-=-M_t\Ind{E_n}$.

    Denoting $A_t = \beta_t I$, we saw that $(M_t^+)^2\preceq A_t^2$ and $(M_t^-)^2\preceq A_t^2$, where
    \begin{align*}
        \norm*{\sum_{s=1}^t A_t^2}_{op} 
        = \norm*{\sum_{s=1}^t \beta_t^2I}_{op}
        =t\beta_t^2.
    \end{align*}
    Thus, by \Cref{lemma: matrix azuma}, the following two events hold w.p. at least $1-\frac{\delta}{3t(t+1)}$:
    \begin{align*}
        &\lambda_{\max}\br*{\sum_{s=1}^t M_s\Ind{E_n}} = \lambda_{\max}\br*{\sum_{s=1}^t M_s^+}< \sqrt{8t\beta_t^2\ln\frac{3dt(t+1)}{\delta}}, \quad\textrm{and,}\\
        &\lambda_{\min}\br*{\sum_{s=1}^t M_s\Ind{E_n}} = -\lambda_{\max}\br*{\sum_{s=1}^t M_s^-}> -\sqrt{8t\beta_t^2\ln\frac{3dt(t+1)}{\delta}}.
    \end{align*}
    Taking the union bound over both events and all $t\ge1$ while  recalling that $\sum_{t\ge1}\frac{1}{t(t+1)}=1$, we get that w.p. at least $1-\frac{2\delta}{3}$,
    \begin{align*}
        \forall t\ge1, \qquad \norm*{\sum_{s=1}^t M_s\Ind{E_n}}_{op} <\sqrt{8t\beta_t^2\ln\frac{3dt(t+1)}{\delta}}.
    \end{align*}

    Now using the fact that $\Pr\brc*{E_n}\geq 1-\frac{\delta}{3}$, we have
    \begin{align*}
        \Pr&\brc*{\exists t: \norm*{\sum_{s=1}^t \br*{X_sX_s^\top-\Sigma_s}}_{op} \geq \sqrt{8t\beta_t^2\ln\frac{3dt(t+1)}{\delta}}} \\
        & \leq \Pr\brc*{E_n, \exists t:\norm*{\sum_{s=1}^t \br*{X_sX_s^\top-\Sigma_s}\Ind{E_n}}_{op} \geq \sqrt{8t\beta_t^2\ln\frac{3dt(t+1)}{\delta}}}  + \Pr\brc*{\bar{E}_n} \\
        & \leq \Pr\brc*{\exists t: \norm*{\sum_{s=1}^t M_s\Ind{E_n}}_{op} \geq \sqrt{8t\beta_t^2\ln\frac{3dt(t+1)}{\delta}}}  + \frac{\delta}{3} \\
        & \leq \frac{2\delta}{3}+ \frac{\delta}{3}\\
        & =  \delta.
    \end{align*}
\end{proof}
\end{document}